\newif\ifdraft\draftfalse
\newif\ificml\icmlfalse

\documentclass{article}

\usepackage{microtype}
\usepackage{graphicx}
\usepackage{caption}
\usepackage{subcaption}
\usepackage{booktabs} 

\usepackage{hyperref}
\usepackage{blindtext}
\usepackage[usenames,dvipsnames,svgnames,table]{xcolor}
\usepackage{xcolor}
\usepackage{mathtools}
\usepackage{subcaption}
\usepackage{amsmath,amssymb,amsthm}

\usepackage[accepted]{icml2021}

\icmltitlerunning{}

\newtheorem{thm}{Theorem}
\newtheorem*{remark*}{Remark}
\newtheorem{lemma}{Lemma}
\newtheorem{definition}{Definition}
\newtheorem{pro}{Proposition}

\newcommand{\E}{\mathbb{E}}

\newcommand{\Reals}{\mathbb{R}}
\newcommand{\Nats}{\mathbb{N}}

\newcommand{\kl}{\mathrm{KL}}

\newcommand{\EE}{\mathbb{E}}
\newcommand{\inv}{^{\raisebox{.ex}{$\scriptscriptstyle-1$}}}

\newcommand{\bI}{\mathbf{I}}
\newcommand{\cN}{\mathcal{N}}

\newcommand{\xhdr}[1]{\vspace{2mm}\noindent{{\bf #1.}}}

\DeclareMathOperator*{\argmin}{arg\,min}

\newcommand{\robust}{robust\xspace}
\newcommand{\robustness}{robustness\xspace}
\newcommand{\hideh}[1]{}

\begin{document}
\ificml
\else
\thispagestyle{plain}
\pagestyle{plain}
\fi

\twocolumn[
\icmltitle{Towards the Unification and Robustness of\\ Perturbation and Gradient Based Explanations
}


\icmlsetsymbol{equal}{*}

\begin{icmlauthorlist}
\icmlauthor{Sushant Agarwal}{to}
\icmlauthor{Shahin Jabbari}{goo}
\icmlauthor{Chirag Agarwal}{goo,equal}
\icmlauthor{Sohini Upadhyay}{goo,equal}\\
\icmlauthor{Zhiwei Steven Wu}{ed}
\icmlauthor{Himabindu Lakkaraju}{goo}
\end{icmlauthorlist}

\icmlaffiliation{to}{David R. Cheriton School of Computer Science, University of Waterloo, Waterloo, ON, Canada}
\icmlaffiliation{goo}{Department of Computer Science, Harvard University, Cambridge, MA, USA}
\icmlaffiliation{ed}{School of Computer Science, Carnegie Mellon University, Pittsburgh, PA, USA}

\icmlcorrespondingauthor{Sushant Agarwal}{sushant.agarwal@uwaterloo.ca}
\icmlcorrespondingauthor{Shahin Jabbari}{jabbari@seas.harvard.edu}

\icmlkeywords{Model Explainability, Local Explainability Methods, Model Interpretability}
\vskip 0.3in

]


\printAffiliationsAndNotice{\icmlEqualContribution} 

\begin{abstract}
As machine learning black boxes are increasingly being deployed in critical domains such as healthcare and criminal justice, there has been a growing emphasis on developing techniques for explaining these black boxes in a post hoc manner. In this work, we analyze two popular post hoc interpretation techniques: SmoothGrad which is a gradient based method, and a variant of LIME which is a perturbation based method. More specifically, we derive explicit closed form expressions for the explanations output by these two methods and show that they both converge to the same explanation in expectation, i.e., when the number of perturbed samples used by these methods is large. We then leverage this connection to establish other desirable properties, such as robustness, for these techniques. We also derive finite sample complexity bounds for the number of perturbations required for these methods to converge to their expected explanation. Finally, we empirically validate our theory using extensive experimentation on both synthetic and real world datasets.\ificml
\footnote{The full technical version of this paper is available at \url{https://arxiv.org/abs/2102.10618}.}\fi
\end{abstract}

\section{Introduction}
\label{sec:intro}

Over the past decade, predictive models are increasingly being considered for deployment in high-stakes domains such as healthcare and criminal justice. However, the successful adoption of predictive models in these settings depends heavily on how well decision makers (e.g., doctors, judges) can understand and consequently trust their functionality. Only if decision makers have a clear picture of the behavior of these models can they assess when and how much to rely on these models, detect potential biases in them, and develop strategies for improving them~\cite{doshi2017towards}. However, the increasing complexity as well as the proprietary nature of predictive models is making it challenging to understand these complex black boxes, thus motivating the need for tools and techniques that can explain them in a faithful and human interpretable manner.

Several techniques have been recently proposed to construct \emph{post hoc explanations} of complex predictive models. While these techniques differ in a variety of ways,
they can be broadly categorized into
\emph{perturbation vs. gradient based techniques}, based on the approaches they employ to generate explanations. For instance, LIME and SHAP~\cite{ribeiro2016should,lundberg2017unified} are called \emph{perturbation based} methods  
because they leverage perturbations of individual instances to construct interpretable local approximations (e.g., linear models), which in turn serve as explanations of individual predictions of black box models. 
On the other hand, SmoothGrad, Integrated Gradients and GradCAM~\cite{simonyan2013saliency, sundararajan2017axiomatic, selvaraju2017grad,smilkov2017smoothgrad} are referred to as \emph{gradient based} methods since they leverage gradients computed at individual instances to explain predictions of complex models. 
Recent research has focused on empirically analyzing the behavior of perturbation and gradient based post hoc explanations. For instance, several works~\cite{ghorbani2019interpretation, slack2019can,dombrowski2019explanations,adebayo2018sanity,alvarez2018robustness} demonstrated that explanations generated using perturbation based techniques such as LIME and SHAP may not be \robust, i.e., the resulting explanations may change drastically with very small changes to the instances. Furthermore,~\citet{adebayo2018sanity} showed that gradient based methods such as SmoothGrad and GradCAM may not generate interpretations that are faithful to the underlying models. 

While several perturbation and gradient based explanation techniques have been proposed in literature and the aforementioned works have empirically examined their behavior, there is very little work that focuses on developing a rigorous theoretical understanding of these techniques and systematically exploring the connections between them. Recently,~\citet{levine2019certifiably}  theoretically and empirically analyzed the \robustness of a sparsified version of  SmoothGrad but their analysis requires several key modifications to the original SmoothGrad, whereas we study SmoothGrad in its original form. Even more recently,~\citet{garreau2020looking} provided closed form solutions for and theoretically analysed Tabular LIME (LIME restricted to tabular data). We study a simpler, non-discretized variant of LIME that benefits by exhibiting several desirable properties, such as being provably robust, unlike the setting in~\citet{garreau2020looking}.
In addition, these works 
do not explore deeper connections between the two classes of techniques. 

In this work, we initiate a study to unify
perturbation and gradient based post hoc explanation techniques. To the best of our knowledge, this work makes the first attempt at establishing connections between these two popular classes of explanation techniques.  More specifically, we make the following key contributions:
\begin{itemize}
\item We analyze two popular post hoc explanation methods -- SmoothGrad (gradient based) and a variant of LIME (perturbation based) for continuous data that we refer to as Continuous LIME, or C-LIME for short. We derive explicit closed form expressions for the explanations output by these methods and demonstrate that they converge to the same output (explanation) in expectation, i.e., when the number of perturbed samples used by these methods is large. 
\item We then leverage this equivalence result to establish other desirable properties of these methods. More specifically, we prove that SmoothGrad and C-LIME satisfy Lipschitz continuity and are therefore \robust to small changes in the input when the number of perturbed samples is large. This work is the first to demonstrate that a variant of LIME is provably robust.

\item We also derive finite sample complexity bounds for the number of perturbed samples required for SmoothGrad and C-LIME to converge to their expected output.

\item Finally, we prove that both SmoothGrad and C-LIME satisfy other interesting properties such as linearity. 
\end{itemize}

We carry out extensive experimentation with synthetic and real world datasets from diverse domains such as online shopping and finance to analyze the behavior of SmoothGrad and C-LIME. Our empirical results not only validate our theoretical claims but also provide other interesting insights. 
We observe that both SmoothGrad and C-LIME need far fewer perturbations (than what our theory predicts) in practice to converge to their expected output  and/or exhibit robustness. SmoothGrad requires even fewer perturbations than C-LIME to be robust and also converges faster than C-LIME.

We also analyze the effects of other parameters such as the variance of the perturbed samples on the convergence as well as \robustness of these methods, and find that smaller values of variance enable these methods to converge faster and exhibit robustness even with fewer perturbed samples. 
\section{Preliminaries}
\label{sec:prelim}
Let us consider a complex function $f:X \rightarrow Y,$ where $X \subseteq \mathbb{R}^d$ for some $d\in\Nats$ and $Y \subseteq \Reals$.
In this section, we provide an overview of two popular post hoc explanation techniques, namely, SmoothGrad and LIME. 
Both SmoothGrad and LIME are local explanation techniques i.e., they explain individual predictions $f(x)$ of a given model $f$. Furthermore, both these methods fall under the broad category of feature attribution methods which  determine the influence of each feature on a given prediction $f(x)$. Below, we describe these methods in detail. We then lay down the setting and assumptions for this work. 
\subsection{SmoothGrad}
\label{sec:sg}
Gradient based explanations are designed to explain predictions $f(x)$ for any $x\in X$, by computing the derivative  of $f(x)$ with respect to each feature of $x$ \cite{ancona2018towards}. 

Vanilla gradient based explanations are often noisy, highlighting random features and ignoring important ones~\cite{adebayo2018sanity}. A convincing explanation for the noise in gradient based saliency maps is that the derivative of $f$ may fluctuate sharply at small scales. Hence, the gradient at any given point is less meaningful than the average of gradients at local neighboring points. This idea has led to SmoothGrad~\cite{smilkov2017smoothgrad}, which in practice reduces the noise in explanations compared to vanilla gradients. 

More concretely, let $S(x)$ (or simply $S$) denote a set of inputs in the neighborhood of $x$. Using $S$, the (empirical) explanation of SmoothGrad for $f$ at point $x$ is defined to be
\begin{equation*}
\text{SG}^f_{S}(x) = \frac{1}{|S|} \sum_{a\in S}\nabla f(a),
\end{equation*}
where $\nabla f$ is the gradient of $f$. When $S$ is drawn from a distribution $P(x)$ (or simply $P$), the expected explanation of SmoothGrad for $f$ at input $x$ can be defined by replacing the sample average with expectation:
\begin{equation*}
    \text{SG}^f_{P}(x) = \EE_{a \sim P}\left[\nabla f(a)\right].
\end{equation*}
Throughout this paper, we use subscripts $P$ and $S$ to distinguish between expected values and empirical averages in our quantities of interest.

\subsection{LIME}
\label{sec:lime}
Another class of explainability models are perturbation based techniques. LIME is a popular perturbation based method that aims to explain the prediction $f(x)$, by learning an interpretable model that approximates $f$ locally around $x$~\cite{ribeiro2016should}.
To obtain a local explanation,  LIME creates perturbed examples in the local neighbourhood of $x$, observes the predictions of $f$ for these examples, and trains an interpretable model on these labeled examples. 

More concretely, let $S$ denote a set of inputs in the neighborhood of $x$ and $\pi: X \times X \to \Reals^{\geq 0}$ a distance metric over $X$. Let $G$ be a class of explanations (or models) and for any $g\in G$, $\Omega(g)$  denote the complexity of $g$ e.g., the complexity of a linear explanation can be measured as the number of non-zero weights. The (empirical) explanation of LIME can be written as 
\begin{equation*}
\text{LIME}^f_{S}(x)=\argmin_{g\in G} \big\{L_x\left(f, g , S, \pi\right) + \Omega\left(g\right)\big\},
\end{equation*} where the loss function $L$ is defined as 
$$L_x\left(f, g, S, \pi\right)  = \frac{1}{|S|}\sum_{a \in S} \pi\left(x, a\right) \left[f(a) - g(a)\right]^2 .$$

When $S$ is drawn from a distribution $P$, the expected explanation of LIME for $f$ at input $x$ can be written by replacing the sample average in the loss function with expectation. We call this quantity $\text{LIME}^f_{P}(x)$.

\begin{remark*}
 The default implementation of LIME has an additional discretization step for the features before optimization. Our definition of LIME here ignores this discretization. 
\end{remark*}

\subsection{Our Setting and Assumptions}
In our setting, we assume $X = \Reals^d$, i.e., we assume the features to be continuous. Different choices of $Y$ in our setting lead to different learning settings. $Y = \Reals$ leads to regression. $Y = [0,1]$ corresponds to (binary or multi-class) classification when $f(x)$ is interpreted as the probability of $f$ belonging to a specific class. 

For any point $x$, in both SmoothGrad and our variant of LIME (discussed below), we assume the sample $S$ in the neighborhood of $x$ is drawn from $\cN(x, \Sigma)$, where $\Sigma = \sigma^2 \bI$ for some $\sigma^2 > 0$. This is a standard choice in practice \cite{garreau2020looking, smilkov2017smoothgrad}.\footnote{Many of our results hold for arbitrary $\Sigma$. We point these out explicitly when we discuss our results.} 

\paragraph{C-LIME.} We use a variant of LIME for continuous features which we refer to as Continuous-LIME (or simply C-LIME). For any given function $f$, input point $x$ and a sample $S$ of inputs in the neighborhood of $x$, the (empirical) explanation of C-LIME can be written as 
$$\text{C-LIME}^f_S(x) = \argmin_{g\in G}\frac{1}{|S|}\sum_{a \in S} \left[f(a) - g(a)\right]^2,$$
where $G$ is the class of linear models. 

We now highlight the main differences between LIME and C-LIME: (i)
C-LIME assumes that the distance metric $\pi$ is a constant function that always outputs 1. Since C-LIME operates on continuous features and uses a Gaussian distribution centered at $x$ to sample perturbations (unlike LIME which samples perturbations uniformly at random), the resulting perturbations are more likely to be closer to $x$ and do not need to be weighted when fitting a local linear model. 
(ii) While LIME allows for a general class of simple explanations (or models) $G$, we restrict ourselves only to linear models for C-LIME since it focuses on continuous features. (iii) Lastly, we exclude the regularizer $\Omega$ from C-LIME i.e., we set $\Omega(g) = 0$ for all $g \in G$. 
Note that the paper that proposes LIME also advocates for enforcing sparsity by first carrying out a feature selection procedure to determine the top K features and then learning the corresponding weights via least squares~\cite{ribeiro2016should}.  See \ificml the full version \else Appendix~\ref{sec:omitted-regular-sparse} \fi for a discussion of the regularised version of LIME.
Finally, for ease of exposition, throughout we assume the output of C-LIME is simply the weights on each feature, and ignore the intercept term. This can be done without loss of generality by centering. Moreover, we are only interested in the learned weights for the features, and not the intercept. For completeness, all of our proofs are written for the case that the intercept is present. 

When clear from context, we refer to the expected output of SmoothGrad and C-LIME for explaining a function $f$ at point $x$ using a Gaussian distribution with mean $x$ and  covariance matrix $\Sigma$ as $\text{SG}^f_{\Sigma}$ and $\text{C-LIME}^f_{\Sigma}$, respectively. Moreover, when it is clear from the context we replace the subscript $S$ to $n$ in all of our quantities of interest to simply emphasize that the size of sample $S$ is $n$.

\section{Equivalence and Robustness}
\label{sec:equiv-robust}
As our first contribution, in Section~\ref{sec:equivalence}, we show that SmoothGrad and C-LIME provide identical explanations in expectation. This establishes a novel connection between gradient based and perturbation based explanation methods, which are often studied independently. Using this connection, in Section~\ref{sec:robust}, we prove that both SmoothGrad and C-LIME are robust, i.e., the explanations provided by these methods for nearby points do not vary significantly.

\subsection{Equivalence} 
\label{sec:equivalence}
As our first result, in Theorem~\ref{thm:equivalence}, we show that the expected output of SmoothGrad and C-LIME are the same for any function at any given input provided that SmoothGrad and C-LIME use the same Gaussian distribution for gradient computation and perturbations, respectively. 
\begin{thm}
\label{thm:equivalence}
Let $f:\mathbb{R}^d \rightarrow \mathbb{R}$ be a function. Then, for any $x\in X$ and any invertible covariance matrix $\Sigma\in \Reals^d \times \Reals^d$
\begin{align*}
    \text{SG}^f_{\Sigma}(x) & = \text{LIME}^f_{\Sigma}(x) = \Sigma^{-1} cov\left(a,f(a)\right) ,
\end{align*}
where $a$ is a random input drawn from  $\cN(x, \Sigma)$, and $cov(a,f(a))$ is a vector with the $i$'th entry corresponding to the covariance of $f(a)$ and $i$'th feature of $a$. 
\end{thm}
\begin{proof}[Proof sketch]
We separately derive closed forms for SmoothGrad and C-LIME. For SmoothGrad we apply a multivariate version of Stein's Lemma \cite{LANDSMAN2008912,liu1994siegel}. The proof for C-LIME uses calculus, and recovers the explanation of C-LIME by differentiation and solving for the solution where the gradient is 0. See \ificml the full version \else Appendix~\ref{sec:omitted-equiv-robust} \fi for more details.
\end{proof}

We point out that Theorem~\ref{thm:equivalence} holds for any covariance matrix and does not require the covariance matrix to be diagonal. 
Furthermore, we note that the closed forms for both SmoothGrad and C-LIME have a nice structure. For a diagonal $\Sigma$, the $i$'th coefficient of $\text{SmoothGrad}^f(x)$ and $\text{C-LIME}^f(x)$ depends only on the covariance of $f$ and the $i$'th feature. In particular, when $\Sigma = \sigma^2 \bI$, then the $i$'th coefficient is simply $cov(f(a),a_i) \sigma^{-2}$. This term captures the dependence of $f$ on the $i$'th feature of the input.

\subsection{Robustness}
\label{sec:robust}
Many interpretability methods come with the drawback that they are very sensitive to the choice of the point where the prediction of the function is going to be  explained~\cite{alvarez2018robustness, ghorbani2019interpretation}. It is hence desirable to have robust explainability methods where two nearby points with similar labels have similar explanations. 

In this section we show that both SmoothGrad and C-LIME are robust. The notion of robustness we use is Lipschitz continuity which is formally defined as follows.

\begin{definition}\label{def:functional_local_lipschitz}
A function $h: \mathbb{R}^{d_1} \rightarrow \mathbb{R}^{d_2}$ for $d_1, d_2\in \Nats$ is \emph{L-Lipschitz}  if there exists a universal constant $L \in \Reals^{> 0}$, such that $\|h(x) - h(x')\|_2 \leq L \|x-x'\|_2$ for all $x, x' \in \mathbb{R}^{d_1}$. 
\end{definition}

We now formally state our robustness result.
\begin{thm}
\label{robust}
Let $f:\mathbb{R}^d \rightarrow \mathbb{R}$ be a function whose gradient is bounded by $\nabla f_{\max}$ and suppose $\Sigma = \sigma^2\bI$. Then  $\text{SG}^f_{\Sigma}$ and $\text{C-LIME}^f_{\Sigma}$ are both L-Lipschitz with $L = \nabla f_{\max}/(2\sigma)$. \end{thm}
\begin{proof}[Proof sketch]
We first prove the Lipschitzness of SmoothGrad
using the Pinkser and data processing inequalities~\cite{divergence}. 
Theorem~\ref{thm:equivalence} then implies that C-LIME is also Lipschitz. See \ificml the full version \else Appendix~\ref{sec:omitted-equiv-robust} \fi for details.
\end{proof}
Theorem~\ref{robust} shows that both SmoothGrad and C-LIME become less robust (i.e., the Lipschitz constants grows) when explaining functions with larger magnitude of gradients, or when the variance parameter $\sigma^2$ used in gradient computation or perturbations decreases. However, the Lipschitz constant is independent of the input dimension $d$.
\section{Convergence Analysis}
\label{sec:finite-sample}
The results in Section~\ref{sec:equiv-robust} prove the equivalence of SmoothGrad and C-LIME and also robustness of these techniques in \emph{expectation} which corresponds to \emph{large sample limits} in practice. Any useful implementation of these techniques is based on finite number of gradient computations or sample perturbations. In this section, we derive sample complexity bounds to examine how fast the empirical estimates for the outputs of SmoothGrad and C-LIME at any given point will converge to the their expected value. This extends the implications of the results in Section~\ref{sec:equiv-robust} to practical implementations of SmoothGrad and C-LIME.

We start by examining how fast the output of SmoothGrad will converge to its expectation.
\begin{pro}
\label{thm:fine-sg}
Let $f:\mathbb{R}^d \rightarrow \mathbb{R}$ be a  function whose gradient is bounded by $\nabla f_{\max}$. Fix $x\in X$, $\epsilon > 0$ and $\delta > 0$. Let $n \geq C(\nabla f_{\max}/\epsilon)^2\ln(d/\delta)$ for some absolute constant $C$. 
Then with probability of at least $1-\delta$, over a sample $S$ of size $n$ from $\cN(x,\Sigma)$, for any $\Sigma\in \Reals^d$, we have that $|\text{SG}^f_{\Sigma}(x)-\text{SG}^f_n(x)\|_{2} \leq \epsilon.$
\end{pro}

We next examine how fast the output of C-LIME will converge to its expectation.
\begin{thm}
\label{thm:finite-lime}
Let $f:\mathbb{R}^d \rightarrow [-1,1]$ be a  function. Fix $x\in X$, $\epsilon > 0$ and $\delta > 0$. Let $S$ denote a sample of size $n$ from $\cN(x,\Sigma)$ for $\Sigma = \sigma^2\bI$ where $$n \geq C 
\frac{d\ln\left(\frac{d}{\delta}\right)}{\min(\epsilon \sigma^2 , \epsilon \sigma^3/\|x\|_2, \|x\|_2, 1/\sigma^2)^2 },$$ for some absolute constant $C$. Then with probability of at least $1-\delta$,  
$\|\text{C-LIME}^f_{\Sigma}(x)-\text{C-LIME}^f_n(x)\|_{2} \leq \epsilon.$
\end{thm}

\begin{proof}[Proof sketch]
First observe that we can write the output of C-LIME both in expectation and in finite sample using the closed-form solution of ordinary least square as follows
\begin{align*}
    &\left\|\text{C-LIME}^f_{\Sigma}(x)-\text{C-LIME}^f_n(x)\right\|_{2} = \\&\left\|\EE[(a a^\top)]\inv\EE[a f(a)]-(\frac{1}{n}\sum_{b\in S} b b^\top)\inv (\frac{1}{n}\sum_{b\in S} b f(b))\right\|_2,
\end{align*}
where the expectations are with respect to $a\sim\cN(x,\Sigma)$ and we use $b$ to index a sampled data point in a sample $S$ of size $n$. By algebraic manipulation and applying Cauchy-Schwartz and triangle inequalities, the term above is bounded by
\begin{align*}
&\left\|\EE\left[a a^\top\right]\inv\right\|_{2}
\left\|\EE\left[a f(a)\right]-\frac{1}{n}\sum_{b\in S} b f(b)\right\|_{2}+\\&
\left\|\EE\left[a a^\top\right]\inv-\frac{1}{n}\left(\sum_{b\in S} b b^\top\right)\inv\right\|_{2}
\left\|\frac{1}{n}\sum_{b\in S}b f(b)\right\|_{2}.
\end{align*}
Therefore, it suffices to bound each of the 4 terms of the above equation separately. 
We show that the first term is bounded by $1/\sigma^2$ using Weyl's inequality. We then show that, with high probability, the second term is bounded  by
$\epsilon/(2\sigma^2)$ using Union bound, Sub-Gaussian and Chernoff concentration inequalities. 
By applying the Weyl'sm Cauchy-Schwartz inequalities, Bernstein inequality in the sub-exponential case for matrices~\cite{Tropp12} and covariance estimation techniques~\cite{koltchinskii2017}, we show that, with high probability, the third term is bounded by $\epsilon/(4\|x\|_2)$. Finally, we show that the last term is, with high probability, bounded by $2\|x\|_2,$ by using Union, Chernoff and Sub-Gaussian concentration bounds as well as Cauchy-Schwartz and triangle inequality.
Multiplying the 4 bounds and applying a Union bound, we witness the theorem's claim. See \ificml the full version \else Appendix~\ref{sec:omitted-finite-sample} \fi 
for details.
\end{proof}

Fixing $\sigma^2$ and $x$, the bound in Theorem~\ref{thm:finite-lime} has the standard $1/\epsilon^2$ dependency on the error parameter 
$\epsilon$ and $\ln(1/\delta)$ dependency on the probability of failure $\delta$. Fixing other parameters, the sample complexity increases as either $\sigma^2$ or 
$\|x\|_2$ approach 0 or grow larger and larger. In the large regime, the growth in the sample complexity is in line with the intuition that accurate estimates under higher variance scenarios require more samples. In the small regime, in our analysis, the bound on the norm of the inverse of the product matrices will grow with a rate that is proportional to $\sigma^2$ or $1/\|x\|_2$ causing the growth in the sample complexity. We empirically study this dependency in Section~\ref{sec:exp}.

\section{Additional Properties}
\label{sec:desiderata}
In this section we study additional properties that are satisfied by both SmoothGrad and C-LIME. We defer all the omitted proofs of this section to \ificml the full version. \else Appendix~\ref{sec:omitted-desiderata}.\fi

The first property that we study is linearity.
\begin{pro}[Linearity]
\label{pro:linearity}
Fix a covariance matrix $\Sigma\in \Reals^d \times \Reals^d$. For all $f,g:\Reals^d\to \Reals$, $d\in \Nats$ and $\alpha, \beta \in \Reals$ $$\text{SG}^{\alpha f+\beta g}_{\Sigma}=\alpha\text{SG}^{f}_{\Sigma}+\beta\text{SG}^{g}_{\Sigma},  \text{ and, } $$ $$\text{C-LIME}^{\alpha f+\beta g}_{\Sigma}=\alpha\text{C-LIME}^{f}_{\Sigma}+\beta\text{C-LIME}^{g}_{\Sigma}.$$
\end{pro}
Linearity implies that the explanation of a more complex function that can be written as a linear combination of two simpler functions is simply the linear combination of the explanations of each of the simpler functions. This is useful e.g., in situations where computing explanations are computationally expensive and new explanations for linear compositions of functions can be simply derived by linear composition of the previously computed explanations.

The next property we study is proportionality. \begin{pro}[Proportionality]
\label{pro:proportionality}
Let $f: \Reals^d \to \Reals$ be a linear function of the form $f(x) =  \theta^\top x + b$
for $\theta\in \Reals^d$ and $b\in \Reals$. For any $x \in \Reals^d$ and $\Sigma\in \Reals^d \times \Reals^d$
\begin{equation*}
  \text{SG}^{f}_{\Sigma}(x) = \text{C-LIME}^{f}_{\Sigma}(x) = k(x) \theta,
\end{equation*}
for some function $k: \mathbb{R}^d \rightarrow \mathbb{R}$.
\end{pro}
Proportionality implies that when the underlying function is linear both SmoothGrad and LIME provide explanations that are proportional to the weights of the underlying function.
Although explaining the weights of a linear function with another set of weights might appear unnecessary, proportionality can be interpreted as a sanity check for explainability methods.~\citet{garreau2020looking} prove a weaker version of proportionality for C-LIME, where the multiplier $k$ might be different for each feature.

An immediate consequence of proportionality is that, in general, SmoothGrad and C-LIME do not provide sparse explanations (for e.g., when the underlying function $f$ is linear and non-sparse). In practice, sparsity can be promoted by adding a regularizer (i.e., by setting $\Omega(g)$ appropriately in our general setting). We study the regularized version of C-LIME in \ificml the full version. \else Appendix~\ref{sec:omitted-regular-sparse}.\fi

\section{Experiments}
\label{sec:exp}

In this section we evaluate our theoretical findings empirically on synthetic and real world datasets. We analyze the equivalence and robustness of SmoothGrad and C-LIME with respect to the number of perturbations. Finally we assess the sensitivity of these results to varying the hyperparameters such as the variance $\sigma^{2}$ in perturbations.

\subsection{Experimental Setup}
\label{sec:setup}

\textbf{Datasets:~} We  generate a synthetic dataset and use 2 real world classification datasets from the UCI Machine Learning Repository \cite{Dua:2019}. 

\xhdr{1.~Simulated}~ We simulate a 1000 sample classification dataset with a 2 dimensional feature space. We fix $y \in \{0,1\}$ randomly for each instance and sample $x \in \Reals^2$ from $\mathcal{N}(\mu_y, I_2)$ where $\mu_0 = [-1,-1]$ and $\mu_1 = [1,1]$. This results in the class clusters illustrated in Figure \ref{fig:sim_db}. 

\begin{figure}[ht!]
    \centering
    \includegraphics[width=0.45\textwidth]{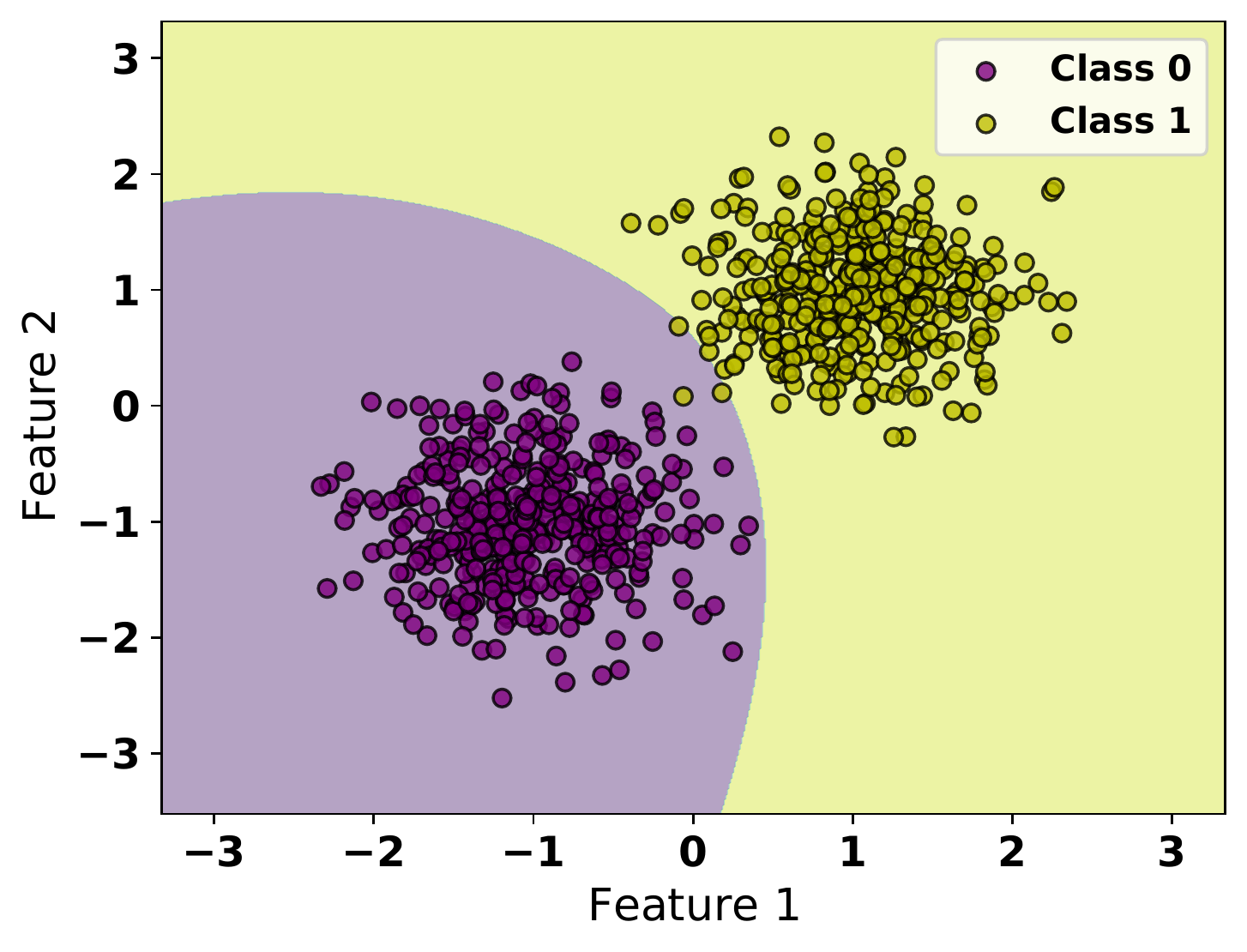}
    \caption{The decision boundary for the  model trained on the simulated data. Each point corresponds to a point in our dataset. Purple and yellow denote the data points with labels 0 and 1, respectively.
    }
    \label{fig:sim_db}
\end{figure}

\xhdr{2.~Bankruptcy}~ This dataset comprises of bankruptcy prediction of Polish companies~\cite{zikeba2016ensemble}.
The input attributes consist of features like net profit, sales and inventory from a pool of 10503 companies. We discard categorical features to align with our theory. As is standard practice when training neural networks, we normalize continuous features to $\mathcal{N}(0,1)$. 
Given the resulting 15 dimensional feature set, the classification task is to predict whether the company in interest will bankrupt or not. 

\xhdr{3.~Online Shopping}~ This dataset comprises 12330 instances of online shopping interactions \cite{sakar2018}. Each sample contains 10 numerical features like the number of pages shoppers visited, time they spend on a page, metrics from Google Analytics and similar. Like with the Bankruptcy dataset, we discard categorical variables and normalize continuous variables, resulting in an 11 dimensional feature space. The target variable for classification is whether an online interaction ends in a purchase or not.

We choose the Bankruptcy and Online Shopping datasets since they contain a large number of real-valued features as assumed by our theory.  

\textbf{Underlying Function:~} For all our experiments, we use a two layer neural network with ELU activation function and 10 nodes per hidden layer.
We follow the standard 80/20 dataset split, \textit{i.e.,} 80\% of the data was used for training the model while 20\% was used for testing. These are the underlying models (functions) that we are explaining in our experiments.
The models are trained using Adam optimizer using a cross-entropy loss function.
Our best performing models achieve a testing accuracy of 99.50\%, 96.30\%, and 99.8\%  using 15, 60, and 100 training epochs for the Simulated, Bankruptcy, and Online Shopping datasets, respectively. We also train models using fewer than the aforementioned training epochs to assess the the impact of  model accuracy on our equivalence and robustness guarantees. 

\textbf{Parameters:} 
Consistent with our theory, for any input point $x$, for both C-LIME and SmoothGrad we generate perturbations from a local neighborhood of $x$ by sampling points from $\cN(x, \sigma^2 \bI)$. We study the effect of the number of perturbations and the value of $\sigma^2$ in our experiments.

\begin{figure}[ht!]
	\centering
	\includegraphics[width=0.95\linewidth]{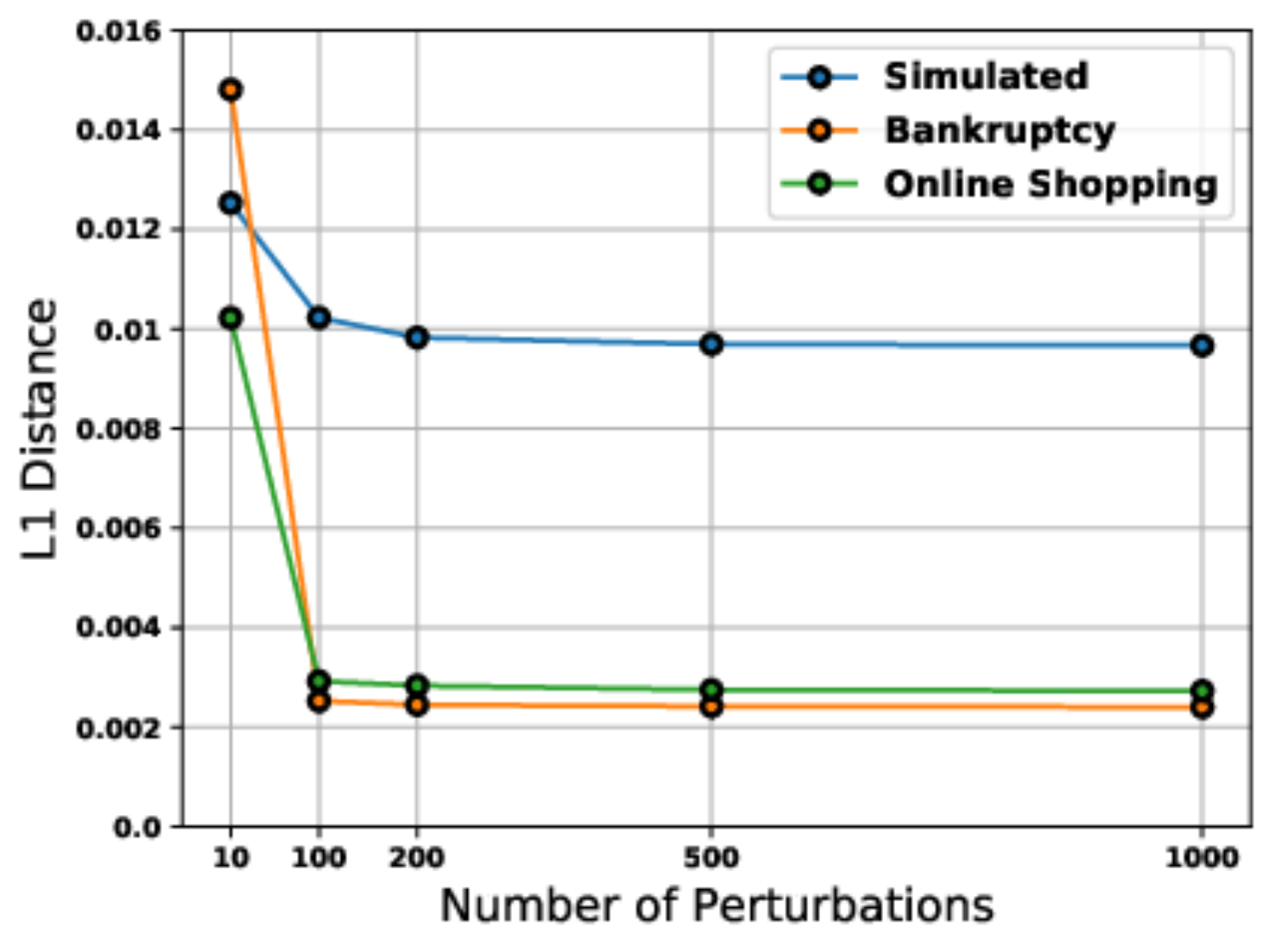}
	\caption{Equivalence plots showing that the $L1$ distance (Y axis) between the explanations of SmoothGrad  and C-LIME decreases as we increase the number of perturbations (X axis). Each curve corresponds to a different dataset.
	}
	\label{fig:equivalence}
\end{figure}

\subsection{Equivalence}
\label{sec:exp_equiv}
To evaluate the equivalence between SmoothGrad and C-LIME, we begin by generating explanations for each instance in the datasets' testing splits using $\sigma^{2}=1$. We probe the effect of varying $\sigma^{2}$ in Section~\ref{sec:sensitivity}. We measure the 
distance between SmoothGrad and C-LIME  explanations for each instance and then average these distances over the entire testing split. We repeat this process for different numbers of perturbations, plotting these average distances versus number of perturbations in Figure~\ref{fig:equivalence}.

\begin{figure*}[ht!]
    \centering
	\begin{subfigure}{0.33\linewidth}
		\centering
    	\includegraphics[width=0.9\linewidth]{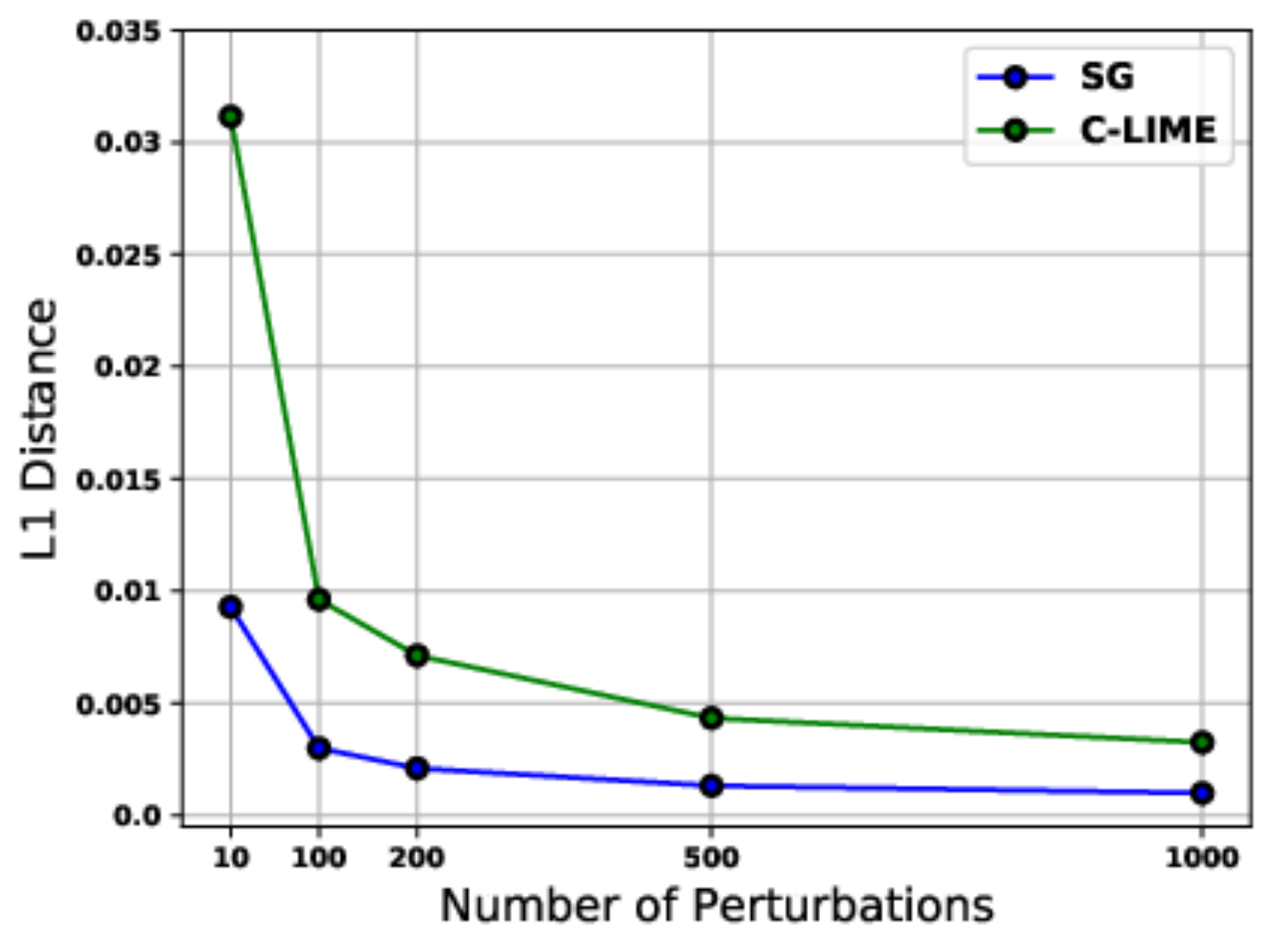}
    	\caption{Simulated}
        \label{fig:sim_robust}
	\end{subfigure}
	\begin{subfigure}{0.33\linewidth}
		\centering
    	\includegraphics[width=0.9\linewidth]{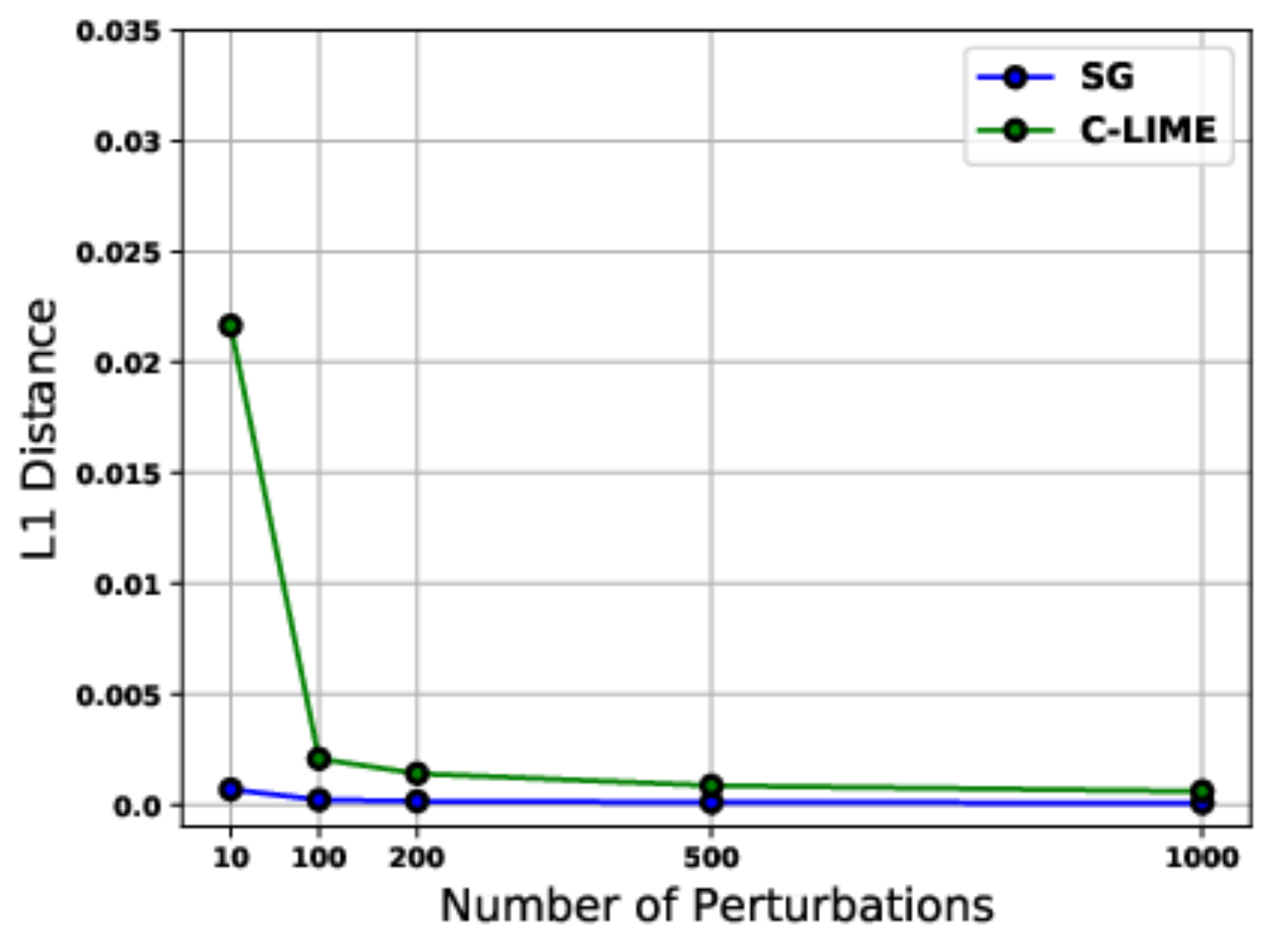}
    	\caption{Bankruptcy}
        \label{fig:bankr_robust}
	\end{subfigure}
	\begin{subfigure}{0.33\linewidth}
		\centering
    	\includegraphics[width=0.9\linewidth]{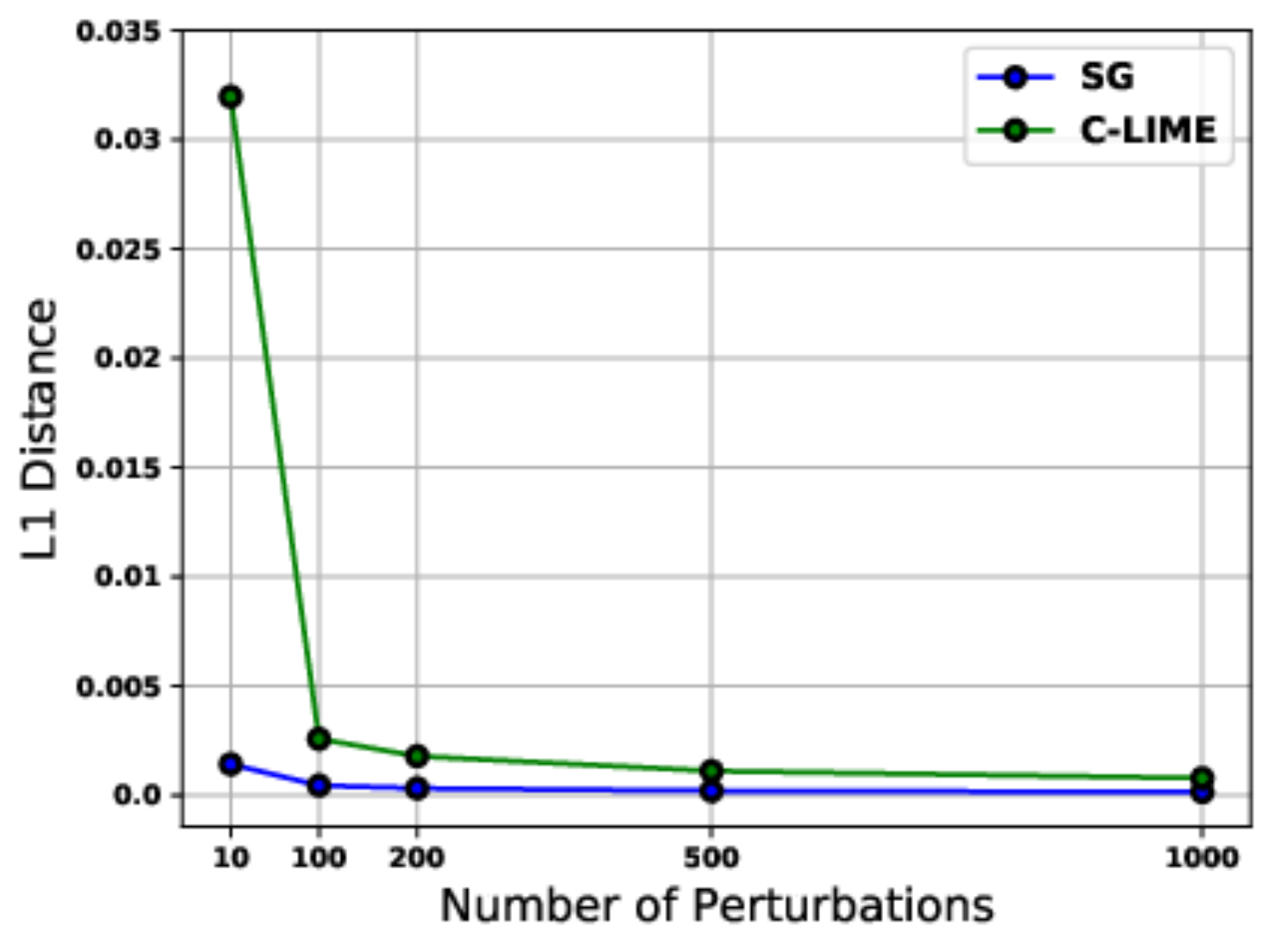}
    	\caption{Shopping}
        \label{fig:shop_robust}
	\end{subfigure}
	\caption{Robustness plots showing the maximum $L1$ distance (Y axis) between the explanations for the original and neighboring inputs averaged over the test data points as a function of number of perturbations used for explanations (X axis). Each plot corresponds to a different dataset. In each plot there are two curves: one for SmoothGrad and another one for C-LIME.
    }
	\label{fig:robustness}
\end{figure*}

We observe that across all three datasets, the average $L1$ distance between the explanations for  SmoothGrad and C-LIME decreases as we increase their respective number of perturbations, supporting equivalence. 
Interestingly, for all the three datasets the equivalence between the two explanation methods is achieved at as low as 100 perturbations. This is significantly lower than the finite perturbation estimates we derive in Proposition \ref{thm:fine-sg} and Theorem \ref{thm:finite-lime}, suggesting that in practice, explanations approach their expected value even with a small number of perturbations.

\subsection{Robustness}
\label{sec:exp_robust}
To evaluate the robustness of  SmoothGrad and C-LIME, we first take each instance $x$ in the testing splits and generate 10 nearby neighbors $x' \sim \mathcal{N}(x,0.01\bI)$. 
We compute explanations for each original instance and its neighbors by perturbations using $\sigma^2 = 1$. For each instance $x$, we compute the distance between the explanation for $x$ and the explanations for each of its neighboring points $x'$. We take the maximum of these distances and then take the average of these maximum distances over the entire testing split. A small value for this average maximum distance suggests that explanations are robust as it implies that the difference between explanations for an instance and its nearby neighbors is small. We compute this average maximum distance for various numbers of perturbations 
and plot them in Figure \ref{fig:robustness}.

The average maximum distance approaches zero across all three datasets, evidencing the robustness of both SmoothGrad and C-LIME. Notice that SmoothGrad appears to be more robust than C-LIME, with the average maximum distance saturating even closer to 0 than C-LIME. Furthermore, SmoothGrad saturates faster than C-LIME at perturbation numbers as small as 200. This suggests that SmoothGrad is more robust than C-LIME for fixed finite perturbations. 

\subsection{Sensitivity Analysis}
\label{sec:sensitivity}
We evaluate the sensitivity of our findings to varying parameters $\sigma^{2}$ and accuracy of the underlying function. 
\paragraph{Sensitivity to $\sigma^2$.}
We begin by evaluating the impact of varying $\sigma^2$ (variance on the perturbations) on our results. We choose to focus on the Bankruptcy dataset, generating the previously described equivalency and robustness plots for $\sigma^2 =$ 0.01, 0.1, and 1, as illustrated in Figure \ref{fig:ablation}. For additional analysis for other dataset refer to \ificml the full version. \else Appendix~\ref{sec:omitted-exp}. \fi 
Notice that SmoothGrad and   C-LIME converge to equivalence faster for smaller $\sigma^2$. Similarly, both SmoothGrad and C-LIME appear to achieve robustness faster for smaller $\sigma^2$. Both of these observations are intuitive as $\sigma^2$ controls the size of the local neighborhood used to generate perturbations. Our theory, on the other hand, predicts that the number of perturbations should increase as either $\sigma^2$ approaches 0 or becomes very large. We suspect that this is due to our style of analysis which requires worst-case bounds on quantities such as the inverse of sampled covariance matrix which hypothetically can grow as $\sigma^2$ approaches $0$.

\begin{figure*}[ht!]
	\begin{subfigure}{0.33\linewidth}
		\centering
    	\includegraphics[width=0.9\linewidth]{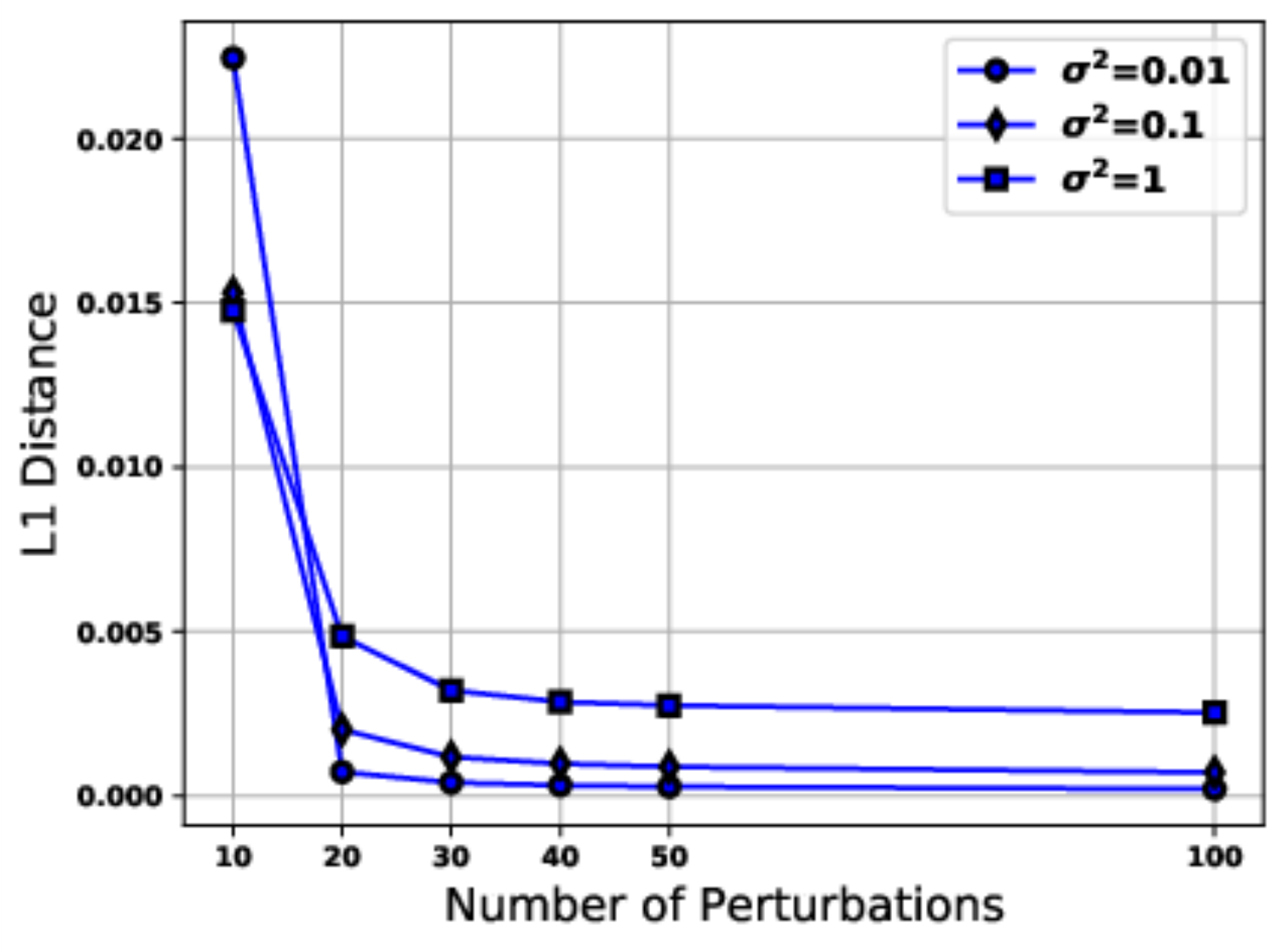}
    	\caption{Equivalence}
        \label{fig:sig_equiv}
	\end{subfigure}
	\begin{subfigure}{0.33\linewidth}
		\centering
    	\includegraphics[width=0.9\linewidth]{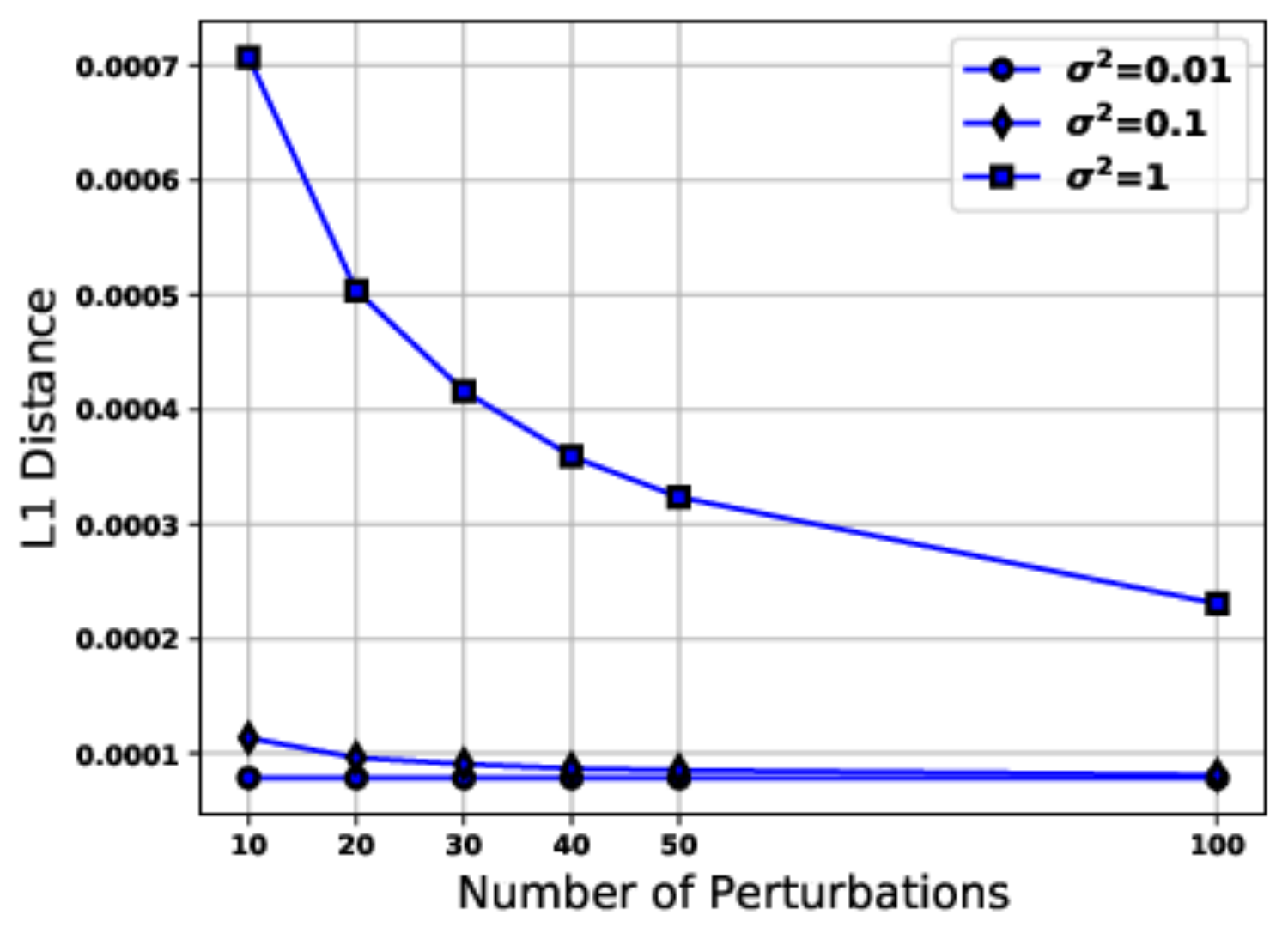}
    	\caption{Robustness of SmoothGrad}
        \label{fig:sig_sg_robust}
	\end{subfigure}
	\begin{subfigure}{0.33\linewidth}
		\centering
    	\includegraphics[width=0.9\linewidth]{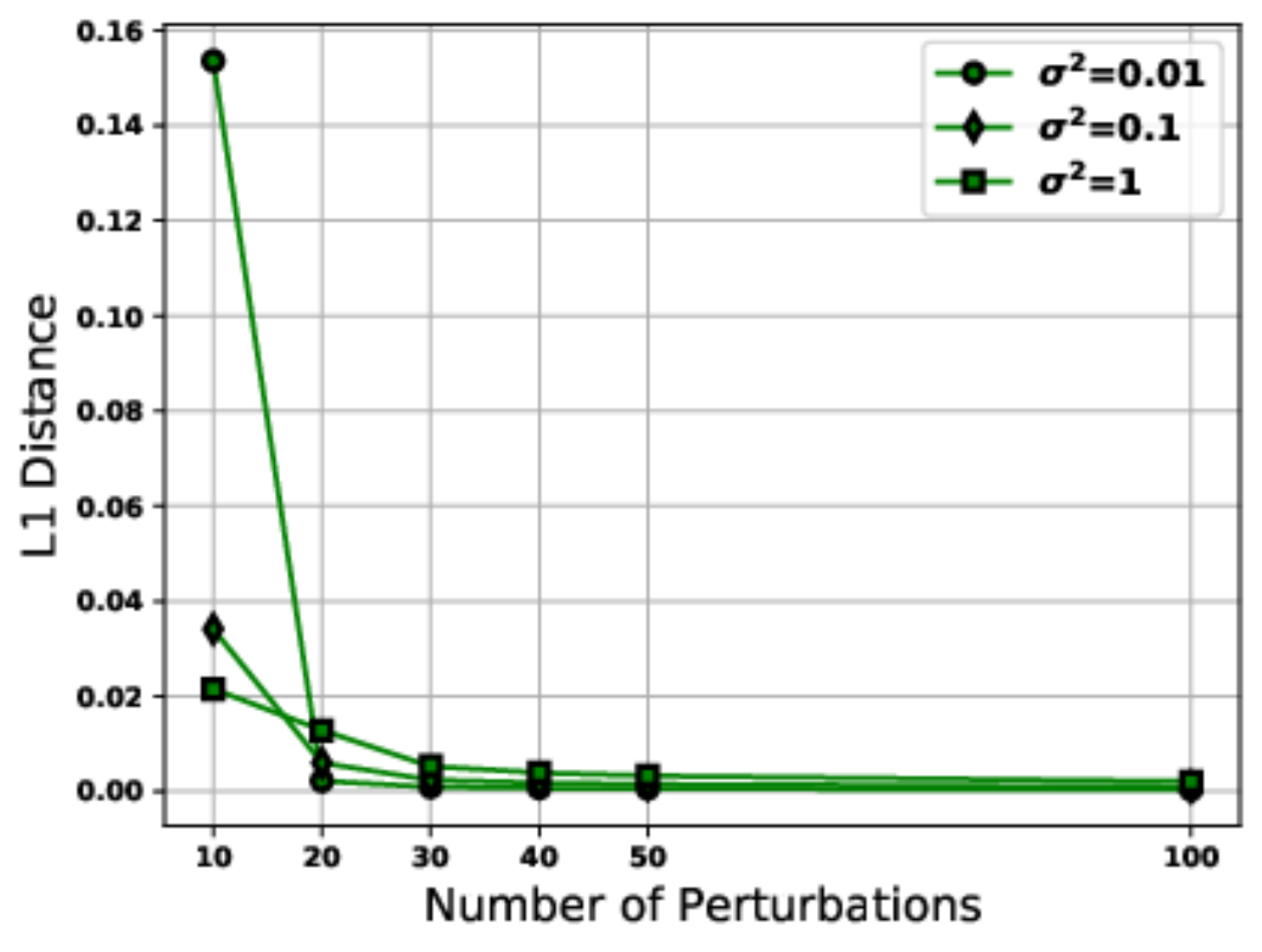}
    	\caption{Robustness of C-LIME}
        \label{fig:sig_lime_robust}
	\end{subfigure}	
	\caption{Equivalence (\ref{fig:sig_equiv}) and robustness plots for SmoothGrad~(\ref{fig:sig_sg_robust}) and C-LIME~(\ref{fig:sig_lime_robust}) for various $\sigma^2$ on the Bankruptcy dataset. In each plot the Y axis corresponds to L1 distance and the X axis corresponds to the number of perturbations.}
	\label{fig:ablation}
\end{figure*}

\paragraph{Sensitivity to Performance of Underlying Function.}
Finally, we analyze whether the performance of the underlying model hinders the equivalence or robustness of SmoothGrad and C-LIME. We modulate model performance by reducing the number of training epochs. We train a model with $86\%$ accuracy using 16 epochs and contrast it to our original model with $96\%$ accuracy after 60 epochs.  Again we choose to focus on the Bankruptcy dataset, generating the previously described equivalency and robustness plots for both  models, as illustrated in Figure \ref{fig:macc}. Interestingly, the convergence rates slow down as the performance of the model becomes worse. See \ificml the full version \else Appendix~\ref{sec:omitted-exp} \fi for more details.

\begin{figure*}[ht!]
	\begin{subfigure}{0.33\linewidth}
		\centering
    	\includegraphics[width=0.9\linewidth]{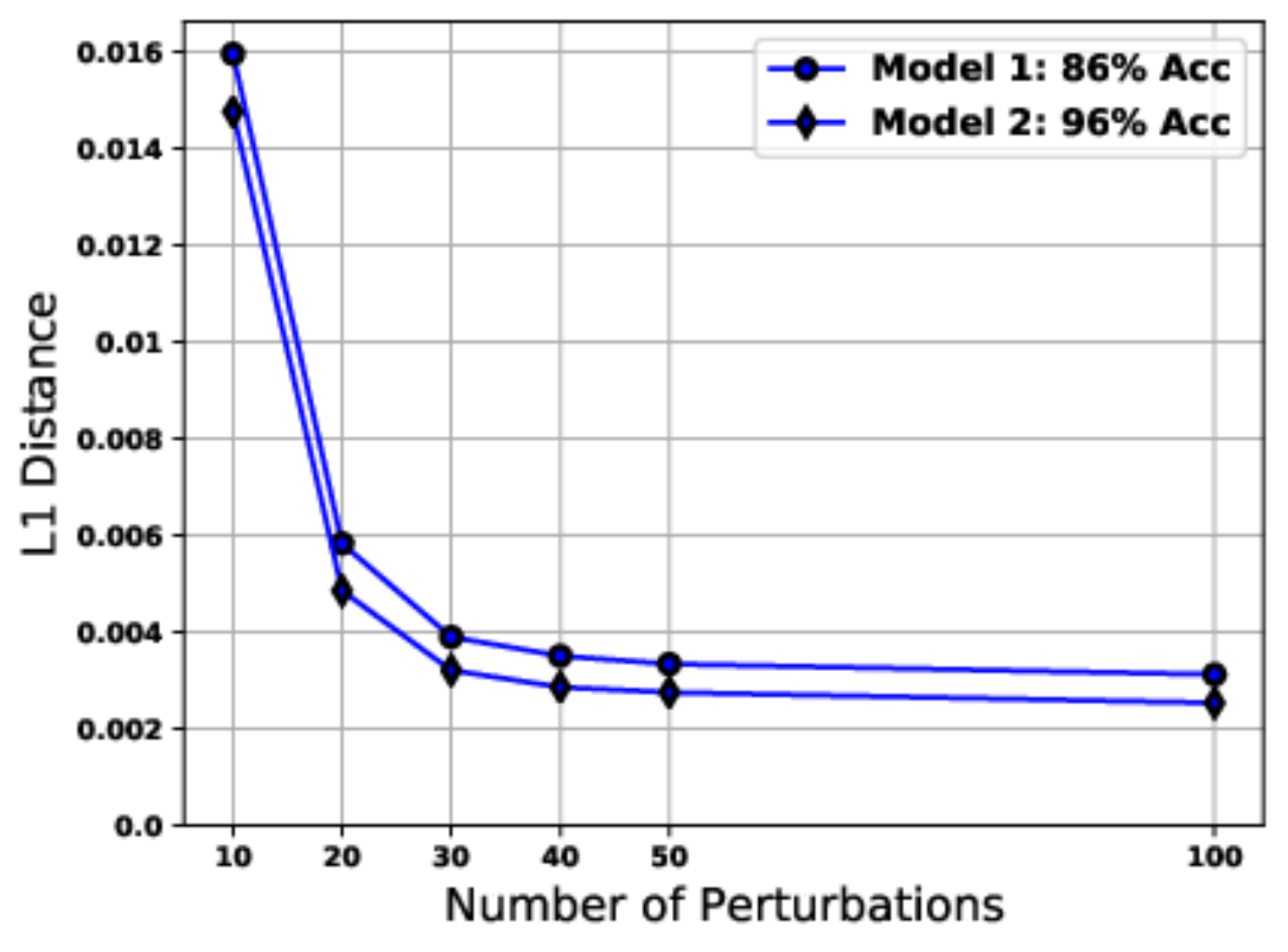}
    	\caption{Equivalence}
        \label{fig:macc_equiv}
	\end{subfigure}
	\begin{subfigure}{0.33\linewidth}
		\centering
    	\includegraphics[width=0.9\linewidth]{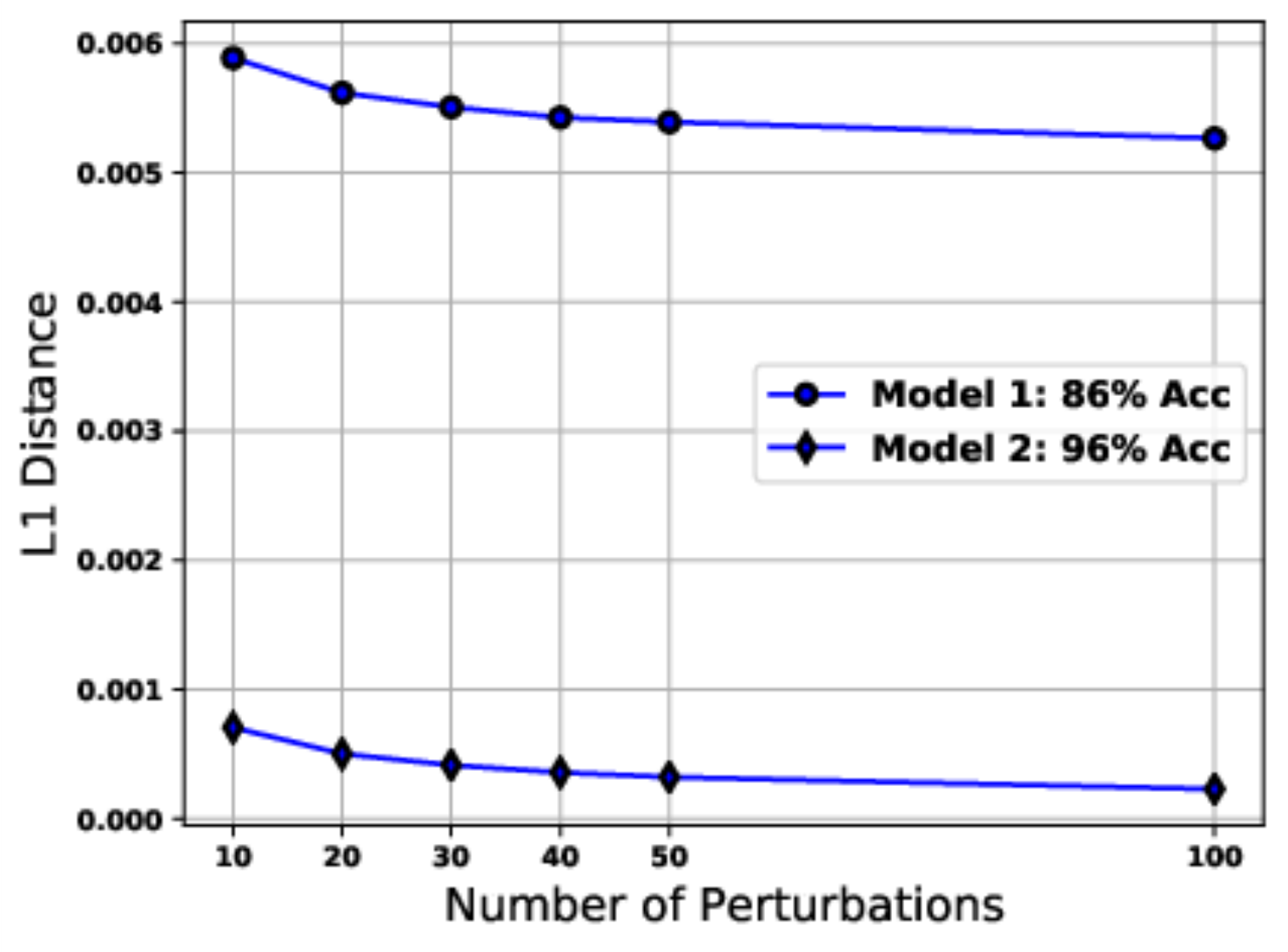}
    	\caption{Robustness of SmoothGrad}
        \label{fig:macc_sg_robust}
	\end{subfigure}
	\begin{subfigure}{0.33\linewidth}
		\centering
    	\includegraphics[width=0.9\linewidth]{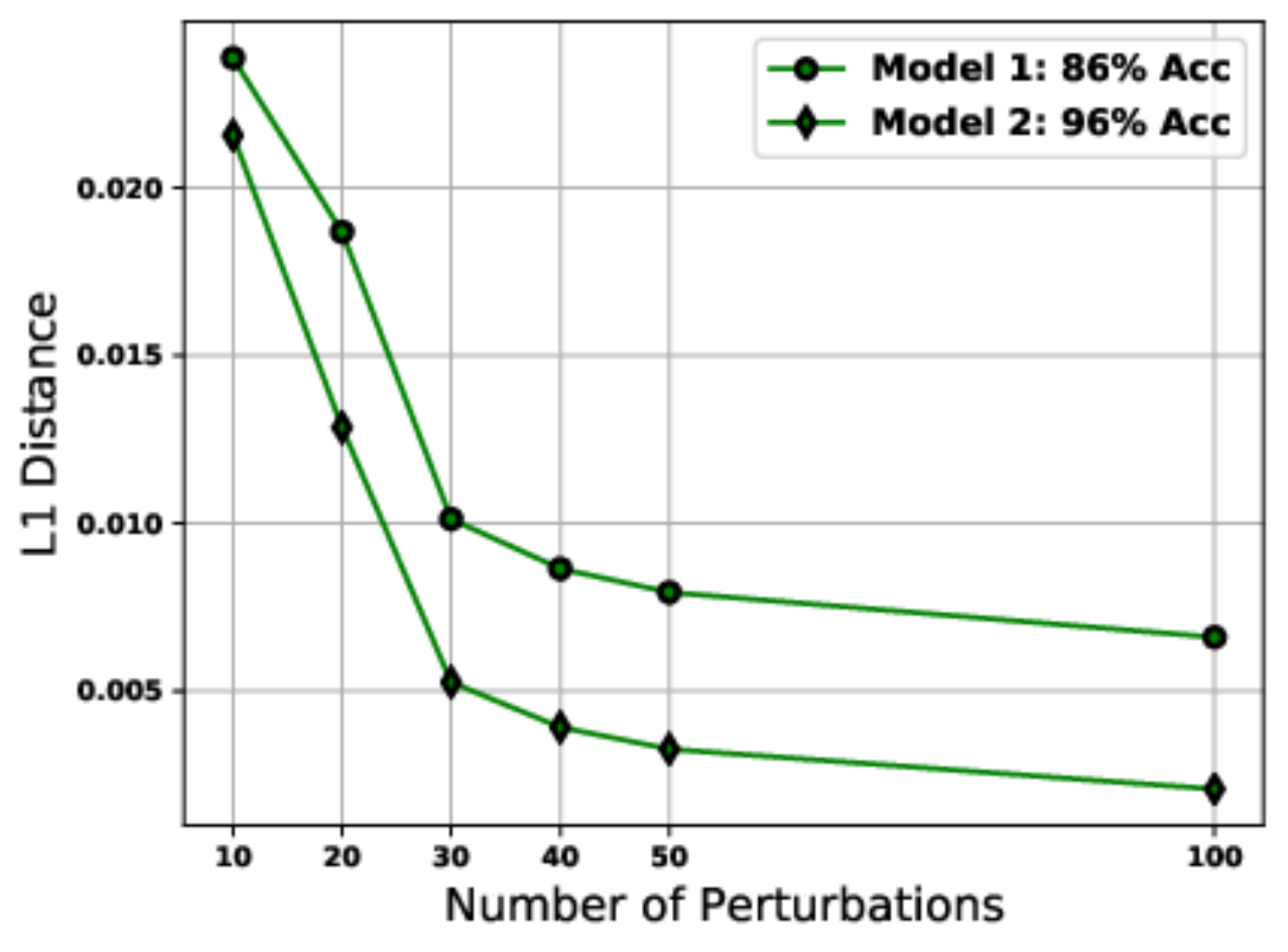}
    	\caption{Robustness of C-LIME}
        \label{fig:macc_lime_robust}
	\end{subfigure}	
	\caption{Equivalence (\ref{fig:macc_equiv}) and robustness plots for SmoothGrad~(\ref{fig:macc_sg_robust}) and C-LIME~(\ref{fig:macc_lime_robust}) for functions with various accuracy on the Bankruptcy dataset. In each plot the Y axis corresponds to L1 distance and the X axis corresponds to the number of perturbations.}
	\label{fig:macc}
\end{figure*}

\section{Related Work}
\label{sec:related}
Interpretability research can be categorized into learning inherently interpretable models, and constructing post hoc explanations. We provide an overview below. 

\xhdr{Inherently Interpretable Models} Many approaches have been proposed to learn inherently interpretable models, for various tasks including classification and clustering. 
To this end, various classes of models such as decision trees, decision lists~\cite{letham15:interpretable}, decision sets~\cite{lakkaraju16:interpretable}, prototype based models~\cite{bien2009classification, kim14:the-bayesian}, and generalized additive models~\cite{lou2012intelligible,caruana15:intelligible} were proposed. However, complex models such as deep neural networks often achieve higher accuracy than simpler  models~\cite{ribeiro2016should}; thus, there has been a lot of interest in constructing post hoc explanations to understand their behavior. 

\xhdr{Post Hoc Explanations} Several techniques have been proposed in recent literature to construct \emph{post hoc explanations} of complex decision models. These techniques 
differ in their access to the complex model (i.e., black box vs. access to internals), scope of approximation (e.g., global vs. local), search technique (e.g., perturbation-based vs. gradient-based), and basic units of explanation (e.g., feature importance vs. rule based). In addition to LIME~\cite{ribeiro2016should} and SHAP~\cite{lundberg2017unified}, there are several other \emph{model-agnostic}, \emph{local explanation} approaches that explain individual predictions of black box models such as Anchors, BayesLIME and BayesSHAP~\cite{ribeiro2018anchors,slack2020i,koh2017understanding}. Several of these approaches rely on input perturbations to learn interpretable local approximations. 

Other local explanation methods including SmoothGrad have been proposed to compute \emph{saliency maps} which capture local feature importance for an individual prediction by computing the gradient at that particular instance~\cite{simonyan2013saliency, sundararajan2017axiomatic, selvaraju2017grad,smilkov2017smoothgrad}. 
There has also been recent work on constructing \emph{counterfactual explanations} which capture what changes need to be made to a given instance in order to flip its prediction~\cite{wachter2017counterfactual,ustun2019actionable,MACE,FACE,looveren2019interpretable, Barocas_2020, karimi2020algorithmic, karimi2020causal}. Such explanations can be leveraged to provide recourse to individuals negatively impacted by algorithmic decisions.
An alternate approach is to construct \emph{global explanations} for summarizing the complete behavior of any given black box by approximating it using interpretable models~\cite{lakkaraju19:faithful,bastani2017interpretability,kim2018interpretability}. 

\xhdr{Analyzing Post Hoc Explanations} Recent work has shed light on the downsides of post hoc explanation techniques. For instance, \citet{rudin2019stop} argued that post hoc explanations are not reliable, as these explanations are not necessarily faithful to the underlying models and present correlations. There has also been recent work on empirically exploring vulnerabilities of black box explanations~\cite{adebayo2018sanity,slack2019can,lakkaraju2020how,rudin2019stop,dombrowski2019explanations}---e.g., \citet{ghorbani2019interpretation} demonstrated that post hoc explanations may not be \robust, changing drastically even with small perturbations to inputs~\cite{alvarez2018robustness}. In addition to the above works, there has also been some recent research that focuses on theoretically analyzing the robustness~\cite{levine2019certifiably,pmlr-v119-chalasani20a}, and other properties~\cite{garreau2020looking} of some of the popular post hoc explanation techniques. However, these works do not attempt to explore deeper connections between different classes of these techniques.

\section{Future Work}
\label{sec:discussion}
We initiate a study on the unification of perturbation and gradient based post hoc explanations,
and pave the way for several promising research directions. It would be interesting to establish connections between other perturbation and gradient based explanations such as SHAP or Integrated gradients. It would also be interesting to study how perturbation and gradient based methods relate to counterfactual explanations. Furthermore, we mainly focused on the analysis of feature attribution methods. It would be exciting to analyze other kinds of explanation methods such as rule based or prototype based methods.

\section*{Acknowledgement}
We sincerely thank Amur Ghose
for invaluable feedback about the proofs of Theorem~\ref{thm:finite-lime} and \ificml Lemmas~3~and~6. \else Lemmas~\ref{LIMEclosed}~and~\ref{LIMEregclosed}.~\fi We also thank Hadi Elzayn for helpful discussion about the proof of Theorem~\ref{thm:finite-lime}. Finally, we thank the anonymous ICML reviewers for their insightful feedback. This work is supported in part by the NSF awards \#IIS-2008461 and \#IIS-2040989, NSF FAI award, Harvard's Center for Research on
Computation and Society, Harvard's Data Science Institute, Amazon, Bayer, and Google. The views expressed are those of the authors and do not reflect the official policy or position of the funding agencies


\bibliographystyle{icml2021}
\bibliography{bib}
\clearpage\newpage
\ificml
\else
\appendix
\section{Proofs of Section~\ref{sec:equiv-robust}}
\label{sec:omitted-equiv-robust}

\subsection{Equivalence}
We separately derive the closed form expressions for the expected outputs of SmoothGrad and C-LIME in Lemmas \ref{SGclosed} and \ref{LIMEclosed}, respectively. 

\paragraph{Closed Form for SmoothGrad.}\label{SG}

We first state Stein's Lemma, that helps us find the closed form expression for the output of SmoothGrad at a point in the large sample limit.

\begin{lemma}[Multivariate version of Stein's Lemma \cite{liu1994siegel}]
\label{lem:stein}
Let $X = \left(X_1, . . . , X_d\right)$ be multivariate normally distributed random variable with mean vector $\mu$ and covariance matrix $\Sigma$. For any function $h: \mathbb{R}^d \rightarrow \mathbb{R}$ such that $\partial h/\partial x_i$ exists almost everywhere, and $\EE \left|\partial h(x)/\partial x_i \right| < \infty, i = 1,...,n$, we write $\nabla h(X) = \left( \partial h/\partial x_1 ,\cdots, \partial h/\partial x_d\right)^\top$. Then 
\begin{equation*}
    cov\left(X, h(X)\right) = \Sigma \,  \EE[ \nabla h(X)].
\end{equation*}{}
\end{lemma}

We can use Lemma~\ref{lem:stein} to derive the closed form expression for the output of SmoothGrad at a point.

\begin{lemma}\label{SGclosed}
Let $f:\mathbb{R}^d \rightarrow \mathbb{R}$ be a function. Then, for any $x\in X$ and invertible covariance matrix $\Sigma\in \Reals^d \times \Reals^d$,
\begin{equation*}
    \text{SG}^f_{\Sigma}(x) = \Sigma^{-1} cov\left(a,f(a)\right),
\end{equation*}
where $a$ is a random input drawn from  $\cN(x, \Sigma)$, and $cov\left(a,f(a)\right)$ is a vector with the $i$'th entry corresponding to the covariance of $f(a)$ and $i$'th feature of $a$.
\end{lemma}{}

\begin{proof}
From Lemma 1, it follows that
\begin{align*}
    &cov\left(a, f(a)\right) = \Sigma \,  \EE[ \nabla f(a)]
   = \Sigma \, \text{SG}^f_{\Sigma}(x)\\&\implies 
   \text{SG}^f_{\Sigma}(x) = \Sigma^{-1} cov\left(a, f(a)\right).
\end{align*}
\end{proof}{}

\paragraph{Closed Form for C-LIME.}\label{LIME}
The below lemma states the closed form expression for the output of C-LIME at a point.

\begin{lemma}\label{LIMEclosed}
Let $f:\mathbb{R}^d \rightarrow \mathbb{R}$ be a function. Then, for any $x\in X$ and invertible covariance matrix $\Sigma\in \Reals^d \times \Reals^d$,
\begin{equation*}
    \text{C-LIME}^f_{\Sigma}(x) = \Sigma^{-1} cov\left(a,f(a)\right), 
\end{equation*}
where $a$ is a random input drawn from  $\cN(x, \Sigma)$, and $cov\left(a,f(a)\right)$ is a vector with the $i$'th entry corresponding to the covariance of $f(a)$ and $i$'th feature of $a$.
\end{lemma}{}

\begin{proof}
We are interested in finding a linear function $W$ that is closest to $f$ in the mean square sense, i.e., we want to minimize 
    $\EE_{a \sim \cN(x,\Sigma)}\left[\left|W(a) - f(a)\right|^2\right]$,
where $W(a) = w^\top a+ b$. The output of LIME is $w$. We want to show that $w = \Sigma^{-1} cov\left(a,f(a)\right).$

Enough to prove the below closed form for $W$.
\begin{equation*}
      W(a) = \EE_{a \sim \cN(x, \Sigma)} [f(a)] + cov(f(a),a) \Sigma^{-1} (a - x).
  \end{equation*}

 We want to minimize objective function $O$ below,
\begin{align*}
   O &= (w^\top a + b - f(a))^2 \\
 &=  (w^\top a+b)^2 + f(a)^2 - 2f(a)(w^\top a+b) \\
    &= (w^\top a)^2 + 2b(w^\top a) + b^2 +  f(a)^2- 2f(a)(w^\top a+b).
\end{align*}
We now take partial derivative of the above expression $O$ w.r.t. $w$ and $b$, and equate with $0$ in both cases, to find the respective values. Observe that,
\begin{equation*}
    (w^\top a)^2 = (w^\top a)(w^\top a) = w^\top aa^\top w.
\end{equation*}
Take the gradient of $(w^\top a)^2$ w.r.t. $w$. This is basically taking derivative w.r.t. $w$ for the form $w^\top A w$. The gradient, in transpose form
is $w^\top(A+A^\top)$. $A$ is symmetric here $(A = aa^\top)$. Hence, derivative of $(w^\top a)^2 = 2 aa^\top w$. Set 
\begin{equation*}
  \frac{\partial O}{\partial w} = 0 \implies 2(aa^\top)w + 2ba - 2f(a)a  = 0.
\end{equation*}
Take expectation over first three terms. Since all expectations are over $\cN(\mu,\Sigma)$, we often drop this dependency for conciseness in the rest of this proof.
 \begin{equation*}
  \EE[aa^\top]w + bx - \EE[f(a)a] = 0.
 \end{equation*}
Set 
\begin{equation*}
  \frac{\partial O}{\partial b} = 0 \implies 2(w^\top a) + 2b - 2f(a) = 0.
\end{equation*}
 Taking expectation, 
$ b = \EE[f(a)] - \EE[w^\top a].$
Substituting value for $b$ into the equation for $w$ gives
\begin{align*}
    &\EE[aa^\top]w + (\EE[f(a)] - \EE[w^\top a])x - \EE[f(a)a] = 0 \implies\\
 &\EE[aa^\top]w - \EE[w^\top a]x - ( \EE[f(a)a] - \EE[f(a)]x ) = 0.
\end{align*}
Observe that,
\begin{align*}
  &\EE[f(a)a] -  \EE[f(a)]x  = cov(a,f(a)) \implies\\ &
 \EE[aa^\top]w - \EE[w^\top a]x -  cov(a,f(a)) = 0.
\end{align*}
Also,
$  \EE[w^\top a] =  w^\top \EE[a] = w^\top x = x^\top w.$
And,
\begin{align*}
   &\EE[w^\top a]x = x\EE[w^\top a] = x x^\top w
    \implies \\& (\EE[aa^\top] - x x^\top)w -  cov(a,f(a)) = 0
    \implies \\& \Sigma w =  cov(a,f(a)).
\end{align*}
Since $\Sigma$ is symmetric positive definite and invertible,
\begin{align*}
     &w = \Sigma^{-1} cov(a,f(a)) \implies\\&
      W(a) = \EE_{a \sim \cN(x,\Sigma)} [f(a)] + cov(f(a),a)\Sigma^{-1} (a-x).
  \end{align*}{} 
  \end{proof}{}

\paragraph{Connecting SmoothGrad and C-LIME}
\begin{proof}[Proof of Theorem \ref{thm:equivalence}]
We derive closed form expressions for SmoothGrad in Lemma \ref{SGclosed} and for C-LIME in Lemma \ref{LIMEclosed}, and it is easy to see these two are equal.
\end{proof}
\subsection{Robustness}

\begin{lemma}\label{lem:robust}
Let $h\colon \Reals^{d_1} \rightarrow \Reals^{d_2}$ for $d_1, d_2 \in \cN$ be a function such that $\|h(x) \|_2 \leq h_{\max}$ for any $x\in \Reals^{d_1}$. Let $g$ be a function defined as
\begin{equation*}
    g(\mu) = \mathbb{E}_{x\sim \mathcal{N}(\mu, \sigma^2 \bI)}\left[h(x)\right].
\end{equation*}
Then $g$ is $(h_{\max}/(2\sigma))$-Lipschitz  w.r.t. $L2$ norm.
\end{lemma}

To prove Lemma~\ref{lem:robust}, we first state the following lemma.

\begin{lemma}\label{ref:pinsker}
Let $X$ and $Y$ be random variables over $\Omega \subseteq C = \{a \in \Reals^d \mid \|a\|_2\leq a_{\max}\text{ and } d\in \cN \}$. Then
\small{
\[
\left|| \mathbb{E}[X] - \mathbb{E}[Y] \right\|_2 \leq a_{\max} \sqrt{\frac{\kl(X\| Y)}{2}}.
\]
}\normalsize
\end{lemma}

\begin{proof}
We will write $X(\cdot)$ and $Y(\cdot)$ to denote the probability densities of the two random variables. The difference in expectations can be written as:
\small{
\begin{align*}
\left\|\mathbb{E}[X] - \mathbb{E}[Y]\right\|_2 &= \left \| \int_{a\in C} a (X(a) - Y(a) ) \right\|_2\\
\tag{Triangle inequality} &\leq \int_{a\in C}  \left\| a \left(X(a) - Y(a)\right )\right\|_2\\
\tag{$\|a\|_2 \leq a_{\max}$} &\leq \int_{a\in C} a_{\max} \left | X(a) - Y(a) \right|\\
 &= a_{\max} \; d_{\text{TV}}(X,Y),
\end{align*}
}\normalsize
where $d_{\text{TV}}$ denotes the total variation distance. The lemma follows from applying the 
Pinsker's inequality: $d_{\text{TV}}(X,Y) \leq \sqrt{\kl(X\| Y) / 2}$.\end{proof}

\begin{proof}[Proof of Lemma~\ref{lem:robust}]
Let $\mu, \nu \in \Reals^{d_1}, \sigma \in \Reals$. The KL divergence of the two multivariate Gaussian distributions can be written as
\[
\kl\left(\mathcal{N}(\mu, \sigma^2 \bI) \| \mathcal{N}(\nu, \sigma^2 \bI)\right) = \frac{\|\mu - \nu\|_2^2}{2 \sigma^2}.
\]
By data processing inequality (see e.g. \citet{divergence}), we know that
\begin{align*}
\kl(g(\mu)\| g(\nu)) &= \kl\left(h \left(\mathcal{N}(\mu, \sigma^2 \bI) \right) \| h\left(\mathcal{N}(\nu, \sigma^2 \bI)\right)\right) \\&\leq \kl(\mathcal{N}(\mu, \sigma^2 \bI) \| \mathcal{N}(\nu, \sigma^2 \bI)).
\end{align*}
By applying Lemma~\ref{ref:pinsker} and inequality above, we get
\begin{align*}
|g(\mu) - g(\nu)| &= \left|  \mathbb{E}_{x\sim \mathcal{N}(\mu, \sigma^2 \bI)}[h(x)] -  \mathbb{E}_{x\sim \mathcal{N}(\nu, \sigma^2 \bI)}[h(x)] \right| \\&\leq h_{\max}  \|\mu - \nu\|_2/(2\sigma),
\end{align*}
which completes the proof.
\end{proof}

\begin{proof}[Proof of Theorem \ref{robust}]
Lemma \ref{lem:robust} establishes the Lipschitzness of SmoothGrad by selecting $h$ to be the gradient of $f$. Due to the equivalence result in Theorem \ref{thm:equivalence}, Lemma \ref{lem:robust} also implies that C-LIME is Lipschitz with the same constant.
\end{proof}
\section{Proofs of Section~\ref{sec:finite-sample}}
\label{sec:omitted-finite-sample}

\subsection{Convergence Analysis of SmoothGrad}
\begin{proof}[Proof of Proposition~\ref{thm:fine-sg}]
By definition of SmoothGrad we have that $SG^f_\Sigma(x) = \EE\left[\nabla f(a)\right]$ where the expectation is with respect to $a \sim \cN(x, \sigma^2 \bI)$. Furthermore, $SG^f_n(x) = \Sigma_{i=1}^{n} \nabla f(a)/n $ where $a \sim \cN(x, \sigma^2 \bI)$. Since $\nabla f$ is bounded, by an application of Chernoff bound and given the choice of number of perturbations $n$ in the statement, we have that in any dimension $i$
$$\big|SG^f_\Sigma(x)[i]-SG^f_n(x)[i]\big| \leq \frac{\epsilon}{\sqrt{d}},$$
with probability of at least $1-\delta/d$. A Union bound then gives us that with probability of at least $1-\delta$, $\left\|SG^f_\Sigma(x)-SG^f_n(x)\right\|_{2}\leq \epsilon$, as claimed.
\end{proof}

\subsection{Convergence Analysis of C-LIME}
\begin{proof}[Proof of Theorem~\ref{thm:finite-lime}]

By definition of C-LIME and using the closed form of the ordinary least square we have that 
$\text{C-LIME}^f_\Sigma(x) = \EE \left[\left(a a^\top\right)\inv a f(a)\right]$,
where $a~\sim \cN(x, \sigma^2 \bI)$. Similarly, we can write down the closed-form of C-LIME using sample $S$ of size $n$ from $\cN(x, \sigma^2 \bI)$ as
\small{
$$
\text{C-LIME}^f_n(x) = \left(\frac{1}{n}\sum_{b \in S} b b^\top\right)\inv \left(\frac{1}{n}\sum_{b \in S}b f(b)\right),
$$
}\normalsize
where $b$ indexes the data points in a sample $S$ of size $n$.

\small{
\begin{align*}
&\left\|\text{C-LIME}^f_\Sigma(x)-\text{C-LIME}^f_n(x)\right\|_{2} \\= &
\left\|\EE\left[(a a^\top)\right]\inv\EE\left[a f(a)\right]-\left(\frac{1}{n}\sum_{b \in S} b b^\top\right)\inv \left(\frac{1}{n}\sum_{b \in S}b f(b)\right)\right\|_2 \\
 = &\|\EE\left[(a a^\top)\right]\inv\EE\left[a f(a)\right] - \EE\left[(a a^\top)\right]\inv\left(\frac{1}{n}\sum_{b \in S}b f(b)\right)\\&+ \EE\left[(a a^\top)\right]\inv\left(\frac{1}{n}\sum_{b \in S}b f(b)\right)\\&-\left(\frac{1}{n}\sum_{b \in S}  b b^\top\right)\inv \left(\frac{1}{n}\sum_{b \in S}b f(b)\right)\|_2 \\
\leq & \left\|\EE\left[a a^\top\right]\inv\right\|_{2}
\left\|\EE\left[a f(a)\right]-\frac{1}{n}\sum_{b \in S}b f(b)\right\|_{2}\\&+\left\|\EE\left[a a^\top\right]\inv-\frac{1}{n}\left(\sum_{b \in S} b b^\top\right)\inv\right\|_{2}
\left\|\frac{1}{n}\sum_{b \in S}b f(b) \right\|_{2},
\end{align*}
}\normalsize
where the inequalities are by  algebraic manipulation and then applying Cauchy–Schwarz and triangle inequalities.

We next bound each of the four terms above separately.
\paragraph{Term 1}
\small{
\begin{align*}
&\left\|\EE\left[a a^\top\right]\inv\right\|_2 \leq \lambda_{\max}\left(\EE\left[a a^\top\right]\inv\right)\\&=\frac{1}{\lambda_{\min}\left(\EE\left[a a^\top\right]\right)} \\&=\frac{1}{\lambda_{\min}\left(\EE\left[aa^\top\right]-\EE\left[a\right] \E\left[a^\top\right]+\EE\left[a\right] \E\left[a^\top\right]\right)}\\
&\leq \frac{1}{\lambda_{\min}\left(\EE\left[a a^\top \right]-\EE\left[a \right] \E\left[a^\top \right]\right)+\lambda_{\min}\left(\EE\left[a\right] \E\left[a^\top\right]\right)}\\&\leq \frac{1}{\sigma^2}.
\end{align*}}\normalsize
The first inequality is due to application of Weyl's inequalities. The last inequality follows from the following two observations. (1) The first term in the denominator is simply the smallest eigenvalue of the covariance matrix of the Gaussian distribution, which is $\sigma^2$. (2) The second term in the denominator corresponds to the smallest eigenvalue of a rank 1 matrix which is 0.

\paragraph{Term 2}
Define $c = a - x$ and $d = b - x$. Hence, both $c$ and $d$ are random variables drawn from $\cN(0, \sigma^2\bI)$. Also, with slight abuse of notation, let $g(y)=f(x+y)$.
\small{
\begin{align*}
    & \left\|\EE\left[a f(a) \right]-\frac{1}{n}\sum_{b \in S}bf(b) \right\|_{2} \\
    &=     \left\|\EE\left[\left(c+x\right) g(c) \right]-\frac{1}{n}\sum_{d \in S}\left(d+x\right)g(d) \right\|_{2}\\
    &=     \left\|\EE\left[c g(c)\right]-\frac{1}{n}\sum_{d \in S}dg(d)  +\EE\left[x g(c)\right]-\frac{1}{n}\sum_{d \in S}xg(d) \right\|_{2} \\
    &\leq     \left\|\EE\left[c g(c)\right]-\frac{1}{n}\sum_{d \in S}dg(d)\right\|_2+ \left\|\EE\left[xg(c)\right]-\frac{1}{n}\sum_{d \in S}xg(d) \right\|_{2} \\    & \leq    \left\|\EE\left[c g(c)\right]-\frac{1}{n}\sum_{d \in S}dg(d)\right\|_2+ \|x\|_2\left\|\EE\left[g(c)\right]-\frac{1}{n}\sum_{d \in S}g(d) \right\|_{2},
    \end{align*}
}\normalsize
where the first inequality is due to triangle inequality and the second one is by Cauchy-Schwarz. Also note that with a slight abuse of notation we summed over $d$ instead of $b$.

Using a Union bound and a Chernoff bound for Sub-Gaussian random variable, with probability of at least $1-\delta/6$, we can bound the first term by $\epsilon \sigma^2/4$ when $n > C d/(\epsilon\sigma^2)^2\ln(d/\delta)$, for some absolute constant $C$. Since $f$ and hence $g$ are bounded in $[-1,1]$, we can use a Chernoff bound to show that when $n> C\|x\|^2_2/(\epsilon\sigma^2)^2\ln(\delta)$, for some absolute constant $C$, then the second term is bounded by $\epsilon \sigma^2/4$ with probability at least $1-\delta/6$.

Applying the union bound then gives us that 
\small{\begin{align*}
    \left\|\EE\left[a f(a) \right]-\frac{1}{n}\sum_{b \in S}bf(b) \right\|_{2}
    &=\frac{\sigma^2\epsilon}{2},
\end{align*}
}\normalsize
with probability of at least $1-\delta/3$.

\paragraph{Term 3}
\small{
\begin{align*}
    &\left\|\EE\left[a  a^\top\right]\inv-\left(\frac{1}{n}\sum_{b \in S} b b^\top\right)\inv\right\|_{2}
      \\= &\left\|\left(\frac{1}{n}\sum_{b \in S} b b^\top\right)^{\inv} \left(\left(\frac{1}{n}\sum_{b \in S} b b^\top\right)\EE\left[a a^\top\right]\inv-\bI\right)\right\|_{2}
     \\ \leq &\left\|\left(\frac{1}{n}\sum_{b \in S} b b^\top\right)^{\inv} \right\|_{2}
     \left\|\left(\frac{1}{n}\sum_{b \in S} b b^\top\right)\EE\left[a a^\top\right]\inv-\bI\right\|_{2},
\end{align*}
}\normalsize
where the last inequality is using the Cauchy-Schwartz inequality. We now proceed to bound each term separately. 

$\newline$
\noindent\textbf{Term 3a} With probability of at least $1-\delta/6$,
\small{
\begin{align*}
& \left\|\left(\frac{1}{n}\sum_{b \in S} b b^\top\right)^{\inv} \right\|_{2}  \leq \lambda_{\max}\left(\left(\frac{1}{n}\sum_{b \in S} b b^\top\right)^{\inv}\right) \\
 & = \frac{1}{\lambda_{\min}\left(\frac{1}{n}\sum_{b \in S} b b^\top\right)} \\
 & = \frac{1}{\lambda_{\min}\left(\frac{1}{n}\sum_{b \in S} b b^\top-\EE\left[a a^\top\right]+\EE\left[a a^\top\right]\right)} \\
& \leq 
\frac{1}{\lambda_{\min}\left(\frac{1}{n}\sum_{b \in S} b b^\top-\EE\left[a a^\top\right]\right)+\lambda_{\min}\left(\EE\left[a a^\top\right]\right)} \\ & \leq \frac{2}{\sigma^2}.
\end{align*}
}\normalsize

The first inequality is by the application of Weyl’s inequalities. The last inequality is the results of the following two observations. (1) Below, we show that 
$\lambda_{\min}\left(\frac{1}{n}\sum_{b \in S} b b^\top-\EE\left[a a^\top\right]\right) \geq -\sigma^2/2$, with probability of at least $1-\delta/6$. (2) Below, we also show that $\lambda_{\min}\left(\EE\left[a a^\top\right]\right) \geq \sigma^2$.

\noindent\textbf{(1)} We want to find a lower bound on $\lambda_{\min}\left(\frac{1}{n}\sum_{b \in S} b b^\top-\EE\left[a a^\top\right]\right)$. First observe that 
\small{
\begin{align*}
    \lambda_{\min}\left(\frac{1}{n}\sum_{b \in S} b b^\top-\EE\left[a a^\top\right]\right)   =  \lambda_{\max}\left(\EE\left[a a^\top\right]-\frac{1}{n}\sum_{b \in S} b b^\top\right).
\end{align*}
}\normalsize
So it suffices to find an upper bound on the maximum eigenvalue of the right term in equality above. Let us define
$$Z = \EE\left[ a a^\top\right]-\frac{1}{n}\sum_{b \in S} b b^\top.$$
Observe that 
$$\EE[Z] = 0 \text{ and }
\EE[Z^p] \leq \frac{p!}{2} \sigma^{p-1} (\sigma\bI)^2 \text{ for }p \in \Nats, p \geq 2.$$
Hence, an application of Bernestein inequality in the subexponential case for matrices~\cite{Tropp12} implies that, with probability of at least $1-\delta/6$,   
\small{
\begin{align*}
\lambda_{\max}\left(\EE\left[a a^\top\right]-\frac{1}{n}\sum_{b \in S} b b^\top\right) \leq \frac{\sigma^2}{2},
\end{align*}
}\normalsize
when $n \geq C\ln(d/\delta)/\sigma^4$ for some absolute constant $C$.

$\newline$
\noindent\textbf{(2)} 
\small{
\begin{align*}
  \lambda_{\min}&\left(\EE\left[a a^\top\right]\right)  \\
  &=   \lambda_{\min}\left(\EE\left[a a^\top\right]-\EE\left[a\right]\EE\left[a^\top\right]+\EE\left[a\right]\EE\left[a^\top\right]\right)\\ &\geq \lambda_{\min}\left(\EE\left[a a^\top\right]-\EE\left[a\right]\EE\left[a^\top\right]\right)+\lambda_{\min}\left(\EE\left[a\right]\EE\left[a^\top\right]\right) \\ & = \sigma^2.
\end{align*}
}\normalsize
The first inequality is due to application of Weyl's inequalities. The last equality corresponds to following two observations: (1) The first term is simply the smallest eigenvalue of the covariance matrix of the Gaussian distribution, which is $\sigma^2$. (2) The second term is  the smallest eigenvalue of a rank 1 matrix which is 0.

$\newline$
\noindent\textbf{Term 3b}
Define 
\small{
$$
\mathbf{E} = \left(\frac{1}{n}\sum_{b \in S} bb^\top\right)-\EE\left[a a^\top\right].$$}\normalsize
Using this notation
\small{
\begin{align*}
    &\left\|\left(\frac{1}{n}\sum_{b \in S} b b^\top\right)\EE\left[a a^\top\right]\inv-\bI\right\|_{2}  \\ &= 
    \left\|\left(\EE\left[a a^\top\right] + \mathbf{E}\right)\EE\left[a a^\top\right]\inv-\bI\right\|_{2} 
     \\ & = \left\|\mathbf{E}\EE\left[a a^\top\right]\inv\right\|_2  \\
     \\ & \leq \left\|\mathbf{E}\right\|_2\left\|\EE\left[a a^\top\right]\inv\right\|_2    \\&  \leq \frac{\epsilon \sigma^2}{8 \|x\|_2}.
\end{align*}
}
\normalsize
First inequality is by Cauchy-Schwartz. To derive the last inequality, we use the following two observations.
(1) By \citet{koltchinskii2017}, with probability of at least $1-\delta/6$, $\|\mathbf{E}\|_2 \leq \epsilon \sigma^4/(8\|x\|_2)$ when 
$$n \geq C \frac{1}{(\frac{\epsilon \sigma^3}{\|x\|_2})^2}\left(d + \ln(\frac{12}{\delta})\right),$$ for some absolute constant $C > 0$.
(2) The other term is equal to Term 1 which we earlier bounded by $1/\sigma^2$.

\paragraph{Term 4} By triangle inequality,
\small{
\begin{align*}
    \left\|\frac{1}{n}\sum_{b \in S}b f(b) \right\|_{2} & =   \left\|\frac{1}{n}\sum_{b \in S}b f(b) - \EE\left[af(a)\right]+\EE\left[af(a)\right] \right\|_{2} 
    \\ & \leq \left\|\frac{1}{n}\sum_{b \in S}b f(b) - \EE\left[a f(a)\right]\right\|_2+\left\|\EE\left[af(a)\right] \right\|_{2} \\
    &= \left\|\frac{1}{n}\sum_{b \in S}b f(b) - \EE\left[a f(a)\right]\right\|_2+\|x\|_{2}.
\end{align*}
}\normalsize

We now bound the first term. Note that this term is identical to Term 2. Hence, when 
$n \geq Cd\ln(d/\delta)/\|x\|_2^2$, it is bounded by $\|x\|_2$ with probability of at least $1-\delta/3$.

This implies that, with probability of at least $1-\delta/3$, 
\small{
\begin{align*}
    \left\|\frac{1}{n}\sum_{b \in S}b f(b) \right\|_{2} & \leq 2\|x\|_2.
\end{align*}}
\normalsize

\paragraph{Putting It All Together}
Multiplying the bounds on all terms and applying a Union bound, with probability of at least $1-\delta$, the $L2$ distance between the empirical average and expected output of C-LIME is at most $\epsilon$. 
\end{proof}

\section{Proofs of Section~\ref{sec:desiderata}}
\label{sec:omitted-desiderata}

\subsection{Linearity}
\begin{proof}[Proof of Proposition~\ref{pro:linearity}]
From Theorem \ref{thm:equivalence}, we know that
\begin{equation*}
    \text{SmoothGrad}^f_\Sigma(x) = \text{C-LIME}^f_\Sigma(x) = \Sigma^{-1}cov(a,f(a)).
\end{equation*}
Hence, it suffices to prove that $\Sigma^{-1}cov(a,f(a))$ is linear. Observe that,
\small{
\begin{align*}
 \Sigma^{-1} &cov(a,(\alpha f+ \beta g)(a)) 
 =\Sigma^{-1} cov(a,\alpha f(a) + \beta g(a))  \\
 &= \Sigma^{-1} cov(a,\alpha f(a)) + \Sigma^{-1} cov(a,\beta g(a)) \\
 &= \alpha \Sigma^{-1} cov(a,f(a)) + \beta \Sigma^{-1} cov(a,g(a)).
\end{align*}
}
\normalsize
Hence, $\Sigma^{-1} cov(a,f(a))$ is linear.
\end{proof}

\subsection{Proportionality}

\begin{proof}[Proof of Proposition~\ref{pro:proportionality}]
From Theorem \ref{thm:equivalence}, we know that
\begin{equation*}
    \text{SmoothGrad}^f_\Sigma(x) = \text{C-LIME}^f_\Sigma(x) = \Sigma^{-1}cov(a,f(a)).
\end{equation*}
Hence, it suffices to prove proportionality for SmoothGrad. Consider a classifier $f$. Set
\begin{equation*}
z = a_0 + a_1 x_1 + \ldots + a_d x_d.
\end{equation*}
This implies
\small{
\begin{align*}
    \nabla f(x_1, \ldots ,x_n)
     &= \left( \frac{\partial f}{\partial x_1}, \ldots, \frac{\partial f}{\partial x_d} \right)
   = \left( \frac{\partial g}{\partial z}\frac{\partial z}{\partial x_1}, \ldots \frac{\partial g}{\partial z} \frac{\partial z}{\partial x_d} \right)
     \\ &= \left( \frac{\partial g}{\partial z} a_1, \ldots,  \frac{\partial g}{\partial z} a_d \right)
     = \frac{\partial g}{\partial z} (a_1, \ldots, a_d).
\end{align*}
}
\normalsize
Hence, for all $(x_1, \ldots ,x_d)$,
\begin{equation*}
    \nabla f(x_1, \ldots ,x_d) = k(x_1, \ldots ,x_d) (a_1, \ldots, a_d),
\end{equation*}
      for some  $k: \mathbb{R}^d \rightarrow \mathbb{R}$, and, for all $(x_1, \ldots, x_d)$, 
\begin{equation*}
    \text{SG}^f_\Sigma(x_1, \ldots, x_d)  
    = k'(x_1, \ldots ,x_d) (a_1, \ldots, a_d),
\end{equation*}
for some  $k': \mathbb{R}^d \rightarrow \mathbb{R}$.
\end{proof}

\section{Regularization}
\label{sec:omitted-regular-sparse}
The below lemma states the closed form expression for the output of C-LIME at a point, when there is an $L2$ regularizer.

\begin{lemma}\label{LIMEregclosed}
Let $f:\mathbb{R}^d \rightarrow \mathbb{R}$ be a function. Then, for any $x\in X$ and invertible covariance matrix $\Sigma\in \Reals^d \times \Reals^d$, and $L2$ regularization parameter
$\lambda \geq 0$,
\begin{equation*}
    \text{C-LIME}^f_{\Sigma}(x) = (\Sigma+\lambda\bI)^{-1} cov\left(a,f(a)\right) ,
\end{equation*}
where $a$ is a random input drawn from  $\cN(x, \Sigma)$, and $cov\left(a,f(a)\right)$ is a vector with the $i$'th entry corresponding to the covariance of $f(a)$ and $i$'th feature of $a$.
\end{lemma}

\begin{proof}
We are interested in finding a linear function $W$ that is closest to $f$ in the mean square sense, plus an added $L2$ regularizer term, i.e., we want to minimise 
\begin{equation}\label{exp}
    \EE_{a \sim \cN(x,\Sigma)}\left[\left|W(a) - f(a)\right|^2\right] + \lambda|w|^2,
\end{equation}
where $W(a) = w^\top a+ b$. The output of LIME is $w$. We want to show that
\begin{equation*}
     w = (\Sigma+\lambda\bI)^{-1} cov\left(a,f(a)\right). 
\end{equation*}
Enough to prove the below closed form for $W$.

\begin{equation*}
      W(a) = \EE_{a \sim \cN(x, \Sigma)} [f(a)] + cov(f(a),a) (\Sigma + \lambda\bI)^{-1} (a - x).
  \end{equation*}
  We need to minimize objective function $O$ below,
\small{
\begin{align*}
   O = &(w^\top a + b - f(a))^2 + \lambda w^2
 \\=  &(w^\top a+b)^2 + f(a)^2 - 2f(a)(w^\top a+b) + \lambda w^2\\
    = &(w^\top a)^2 + 2b(w^\top a) + b^2 +  f(a)^2- 2f(a)(w^\top a+b) + \lambda w^2.
\end{align*}
}
\normalsize
We now take partial derivative of the above expression $O$ w.r.t. $w$ and $b$, and equate with $0$ in both cases, to find the respective values. Observe that,
\begin{equation*}
    (w^\top a)^2 = (w^\top a)(x^\top a) = w^\top aa^\top w.
\end{equation*}
Take the gradient of $(w^\top a)^2$ w.r.t. $w$. This is basically taking derivative w.r.t. $w$ for the form $w^\top A w$. The gradient, in transpose form
is $w^\top(A+A^\top)$. $A$ is symmetric here $(A = aa^\top)$.
Hence, derivative of $(w^\top a)^2 = 2 aa^\top w$. Set 
\begin{equation*}
  \frac{\partial O}{\partial w} = 0 \implies 2(aa^\top)w + 2ba - 2f(a)a + 2\lambda w = 0.
\end{equation*}
Take expectation over first three terms. Since all expectations are over $\cN(\mu,\Sigma)$, we often drop this dependency for conciseness in the rest of this proof.
 \begin{equation*}
  \EE[aa^\top]w + bx - \EE[f(a)a]  + \lambda w = 0.
 \end{equation*}
Set 
\begin{equation*}
  \frac{\partial O}{\partial b} = 0 \implies 2(w^\top a) + 2b - 2f(a) = 0.
\end{equation*}
Taking expectation, 
$
 b = \EE[f(a)] - \EE[w^\top a].$
Substituting value for $b$ into the equation for $w$,
\begin{align*}
    &\EE[aa^\top]w + (\EE[f(a)] - \EE[w^\top a])x - \EE[f(a)a] + \lambda w = 0 \implies\\&
 \EE[aa^\top]w - \EE[w^\top a]x - ( \EE[f(a)a] - \EE[f(a)]x ) + \lambda w = 0.
\end{align*}
Observe that,
\begin{align*}
  &\EE[f(a)a] -  \EE[f(a)]x  =  cov(a,f(a)) \implies \\ &
 \EE[aa^\top]w - \EE[w^\top a]x -  cov(a,f(a)) + \lambda w = 0.
\end{align*}
And,
$
  \EE[w^\top a] =  w^\top \EE[a] = w^\top x = x^\top w.
$
Also,
\begin{align*}
   &\EE[w^\top a]x = x\EE[w^\top a] = x x^\top w
    \implies \\&(\EE[aa^\top] - x x^\top)w -  cov(a,f(a)) + \lambda w = 0
    \implies \\&(\Sigma + \lambda\bI) w =  cov(a,f(a)).
\end{align*}
Since $\Sigma$ is symmetric positive definite and invertible, $\Sigma + \lambda \bI$ is also symmetric positive definite and invertible for $\lambda \geq 0$. This implies that,
\begin{align*}
 &w = (\Sigma + \lambda\bI)^{-1} cov(a,f(a)) \implies \\ &
      W(a) = \EE_{a \sim \cN(x,\Sigma)} [f(a)] + cov(f(a),a)(\Sigma + \lambda\bI)^{-1} (a-x).
  \end{align*}{} 
  \end{proof}{}

\section{Additional Experimental Results}
\label{sec:omitted-exp}
\begin{figure*}[ht!]
	\begin{subfigure}{0.33\linewidth}
		\centering
    	\includegraphics[width=0.9\linewidth]{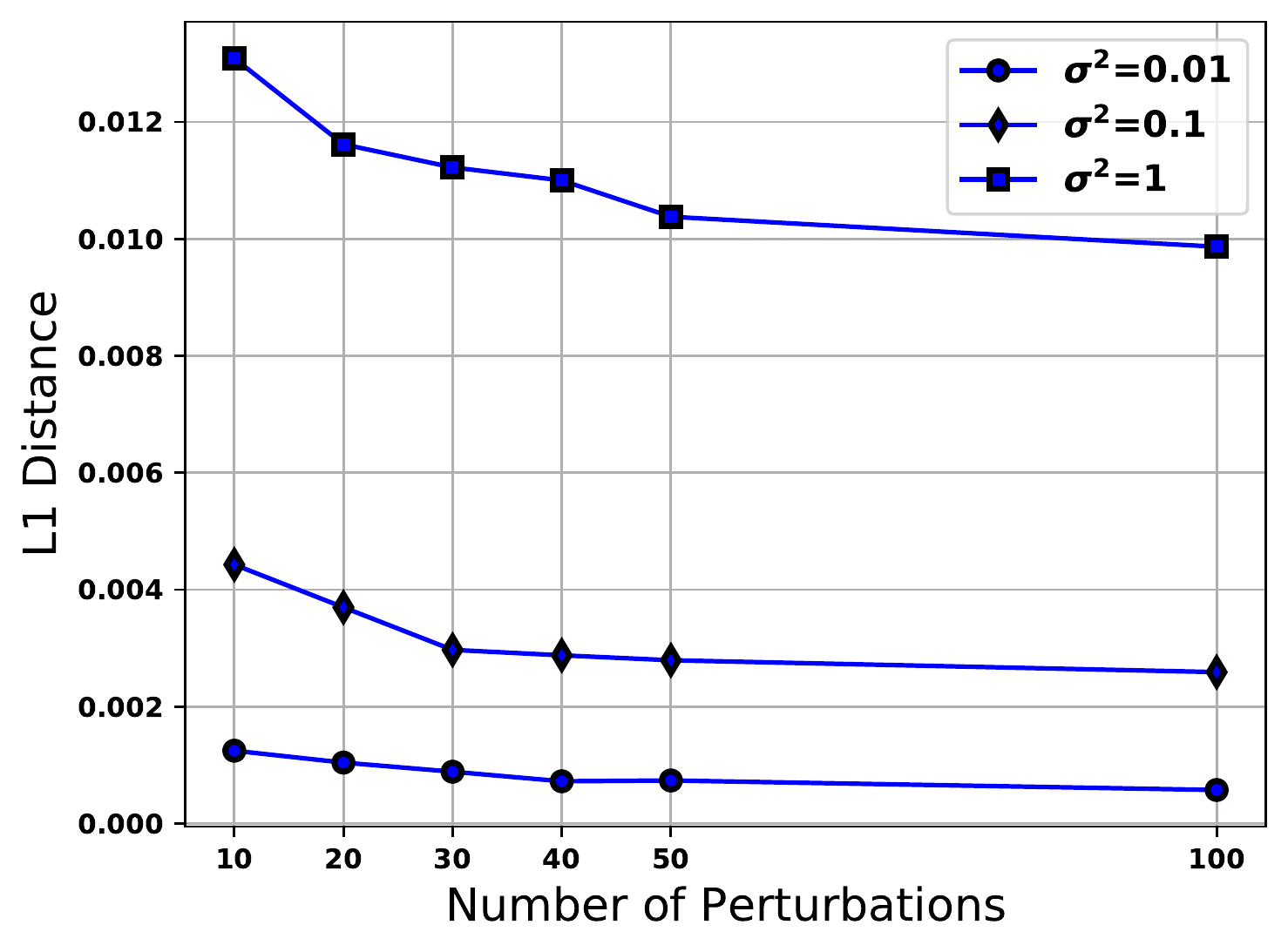}
    	\caption{Equivalence}
        \label{fig:sig_equiv_sim}
	\end{subfigure}
	\begin{subfigure}{0.33\linewidth}
		\centering
    	\includegraphics[width=0.9\linewidth]{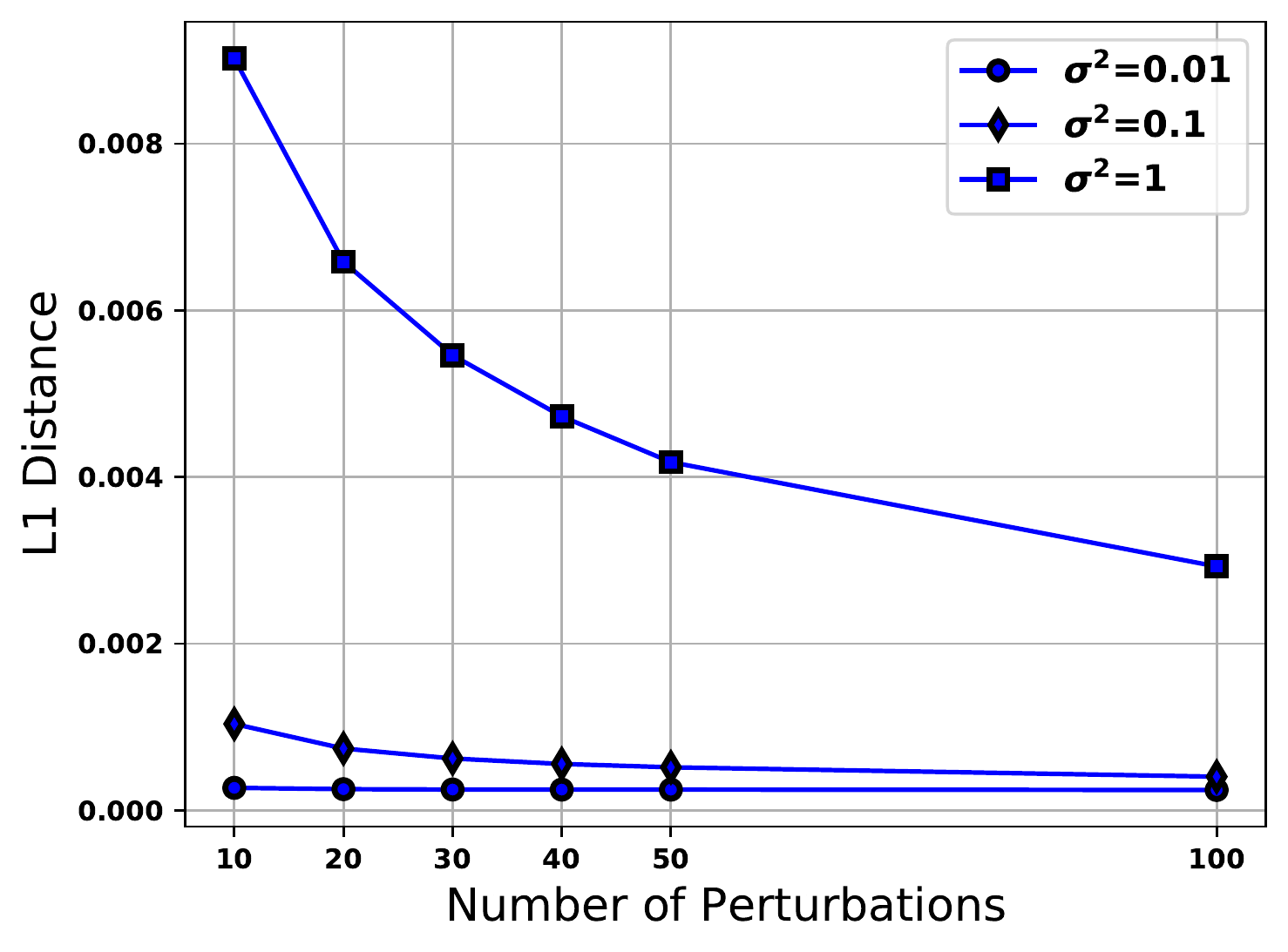}
    	\caption{Robustness of SmoothGrad}
        \label{fig:sig_sg_robust_sim}
	\end{subfigure}
	\begin{subfigure}{0.33\linewidth}
		\centering
    	\includegraphics[width=0.9\linewidth]{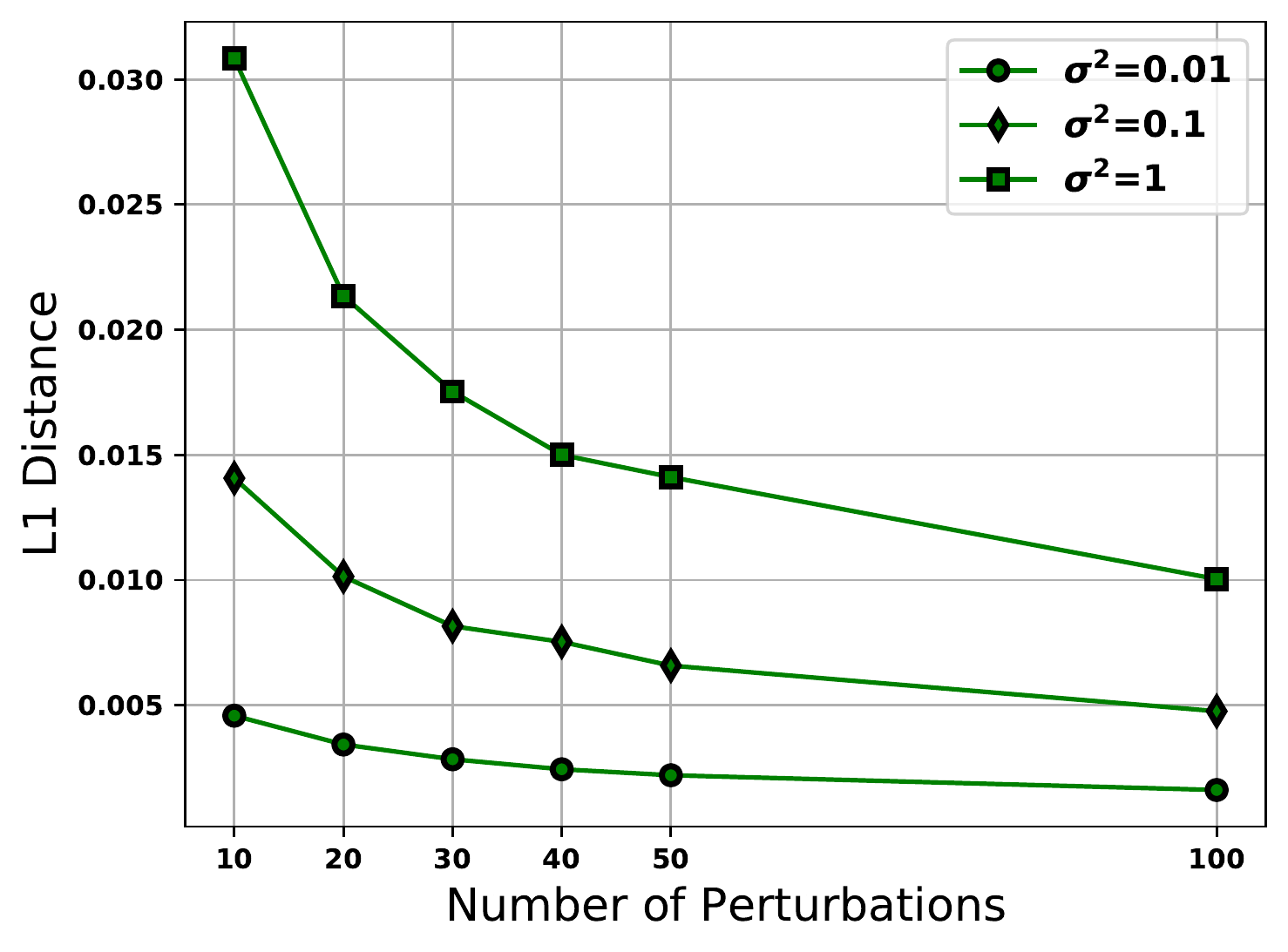}
    	\caption{Robustness of C-LIME}
        \label{fig:sig_lime_robust_sim}
	\end{subfigure}	
	\caption{Equivalence (\ref{fig:sig_equiv_sim}) and robustness plots for SmoothGrad~(\ref{fig:sig_sg_robust_sim}) and C-LIME~(\ref{fig:sig_lime_robust_sim}) for various $\sigma^2$ on the Simulated dataset. In each plot the Y axis corresponds to L1 distance and the X axis corresponds to the number of perturbations.}
	\label{fig:ablation_sim}
\end{figure*}

\begin{figure*}[ht!]
	\begin{subfigure}{0.33\linewidth}
		\centering
    	\includegraphics[width=0.9\linewidth]{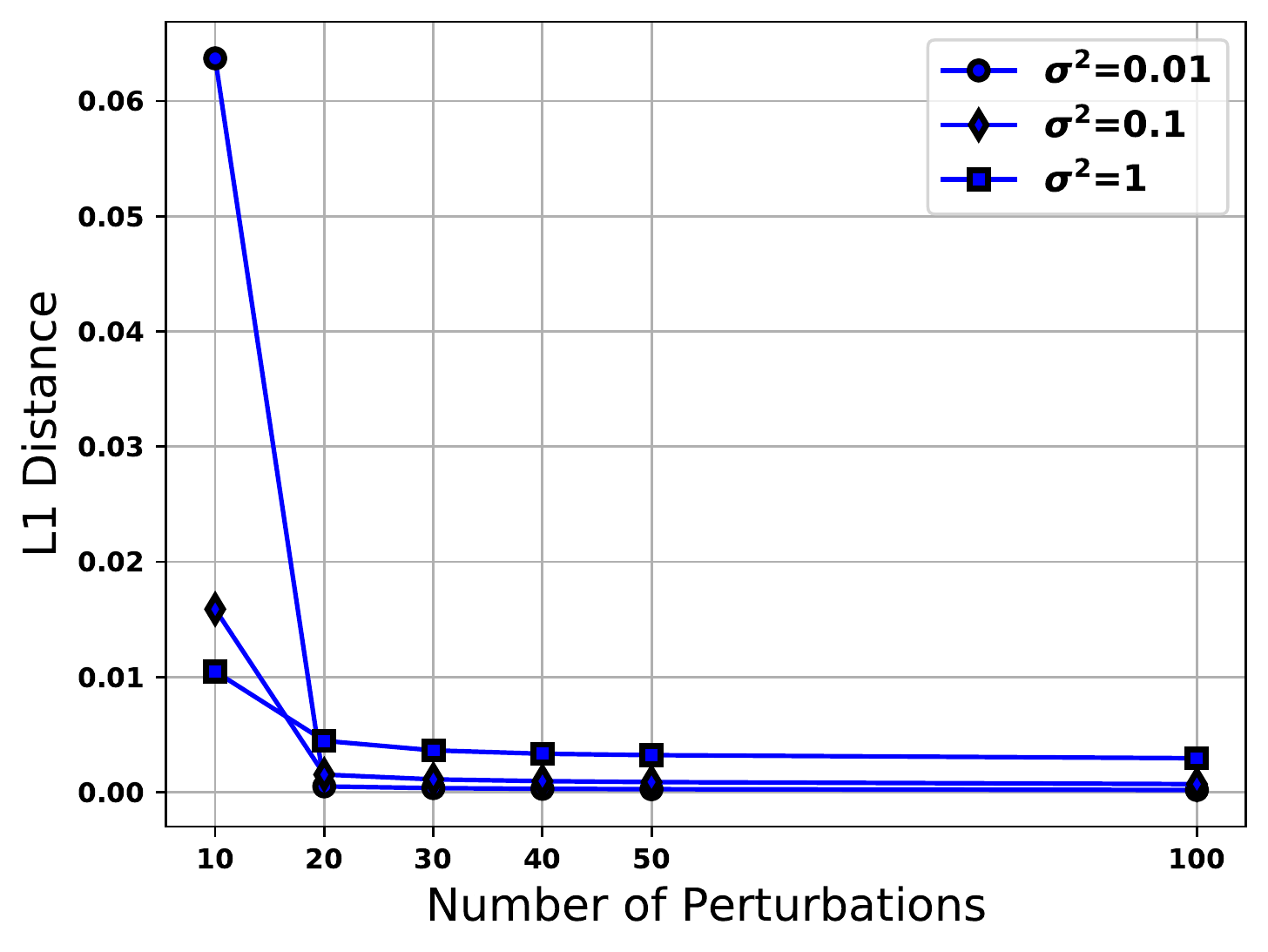}
    	\caption{Equivalence}
        \label{fig:sig_equiv_shopper}
	\end{subfigure}
	\begin{subfigure}{0.33\linewidth}
		\centering
    	\includegraphics[width=0.9\linewidth]{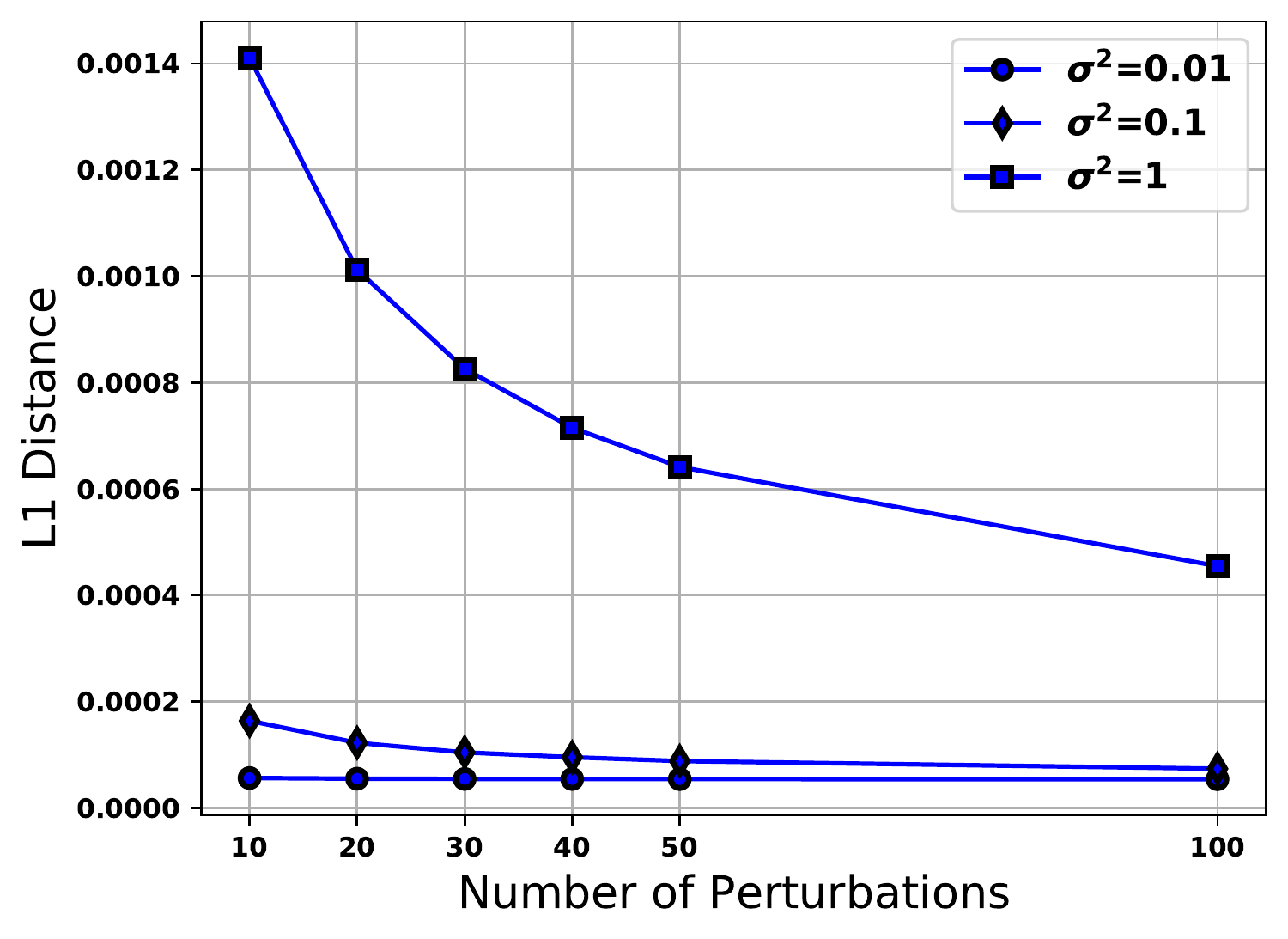}
    	\caption{Robustness of SmoothGrad}
        \label{fig:sig_sg_robust_shopper}
	\end{subfigure}
	\begin{subfigure}{0.33\linewidth}
		\centering
    	\includegraphics[width=0.9\linewidth]{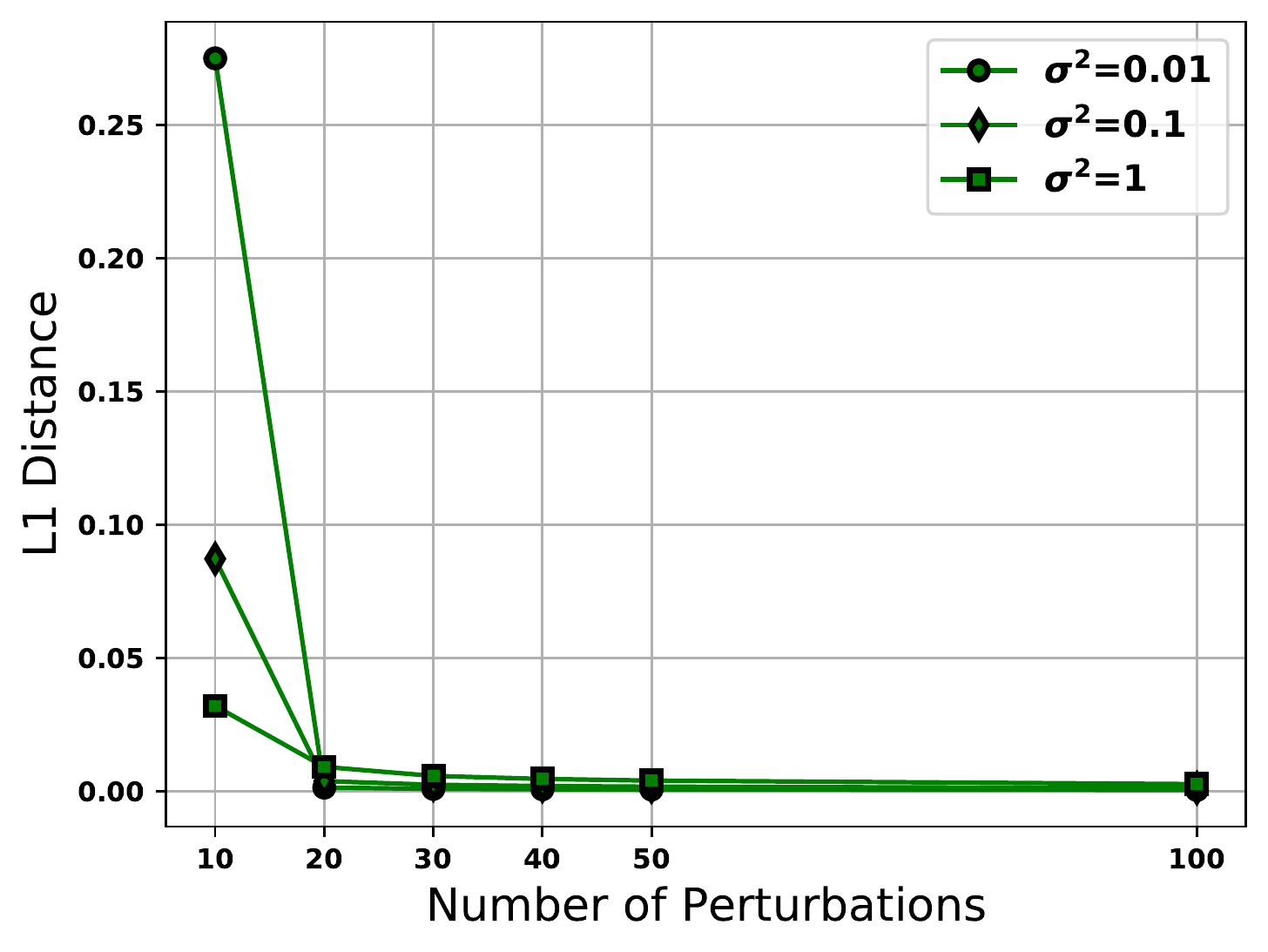}
    	\caption{Robustness of C-LIME}
        \label{fig:sig_lime_robust_shopper}
	\end{subfigure}	
	\caption{Equivalence (\ref{fig:sig_equiv_shopper}) and robustness plots for SmoothGrad~(\ref{fig:sig_sg_robust_shopper}) and C-LIME~(\ref{fig:sig_lime_robust_shopper}) for various $\sigma^2$ on the Online Shopping dataset. In each plot the Y axis corresponds to L1 distance and the X axis corresponds to the number of perturbations.}
	\label{fig:ablation_sim}
\end{figure*}

\begin{figure*}[ht!]
	\begin{subfigure}{0.33\linewidth}
		\centering
    	\includegraphics[width=0.9\linewidth]{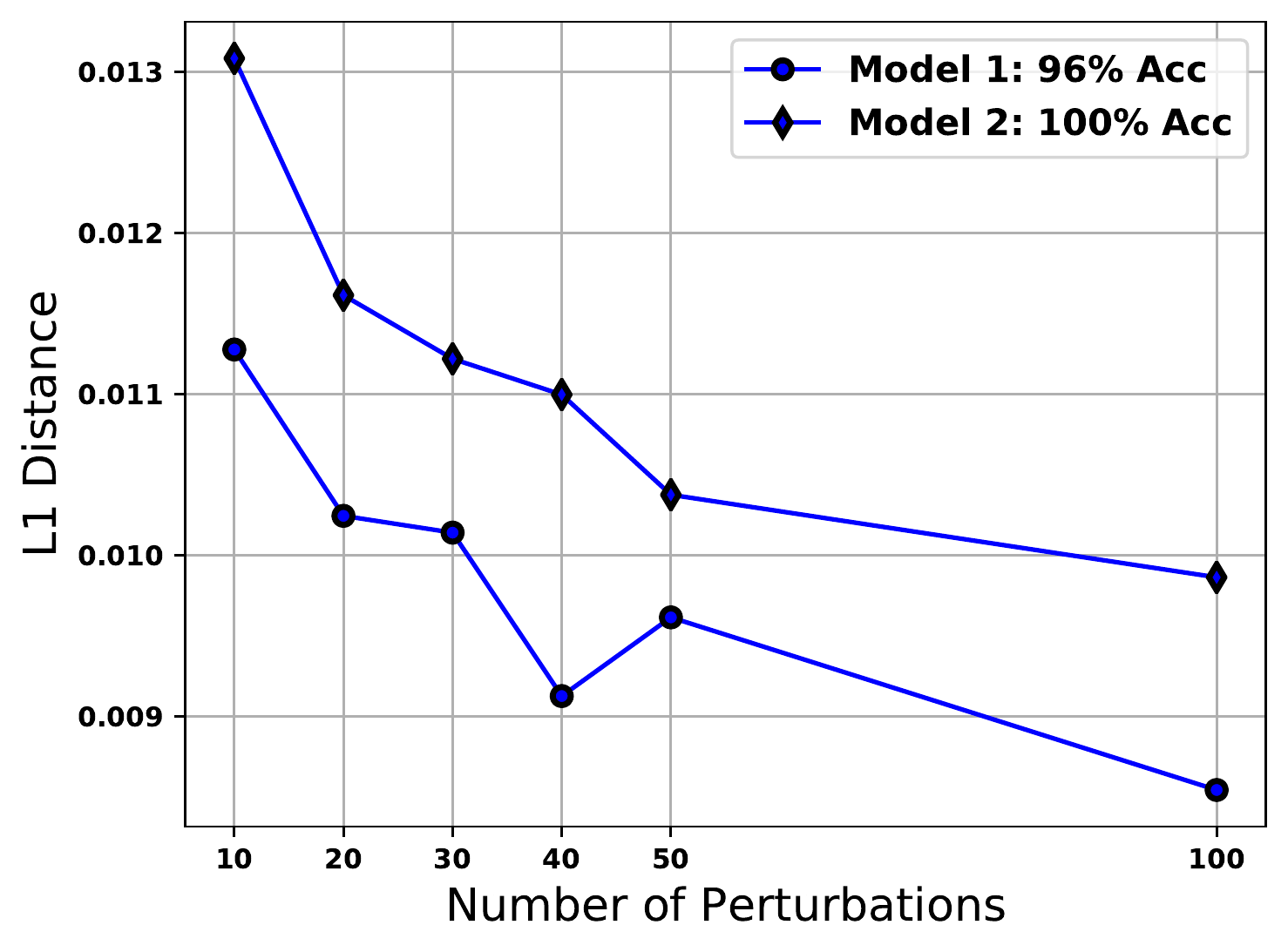}
    	\caption{Equivalence}
        \label{fig:macc_equiv_sim}
	\end{subfigure}
	\begin{subfigure}{0.33\linewidth}
		\centering
    	\includegraphics[width=0.9\linewidth]{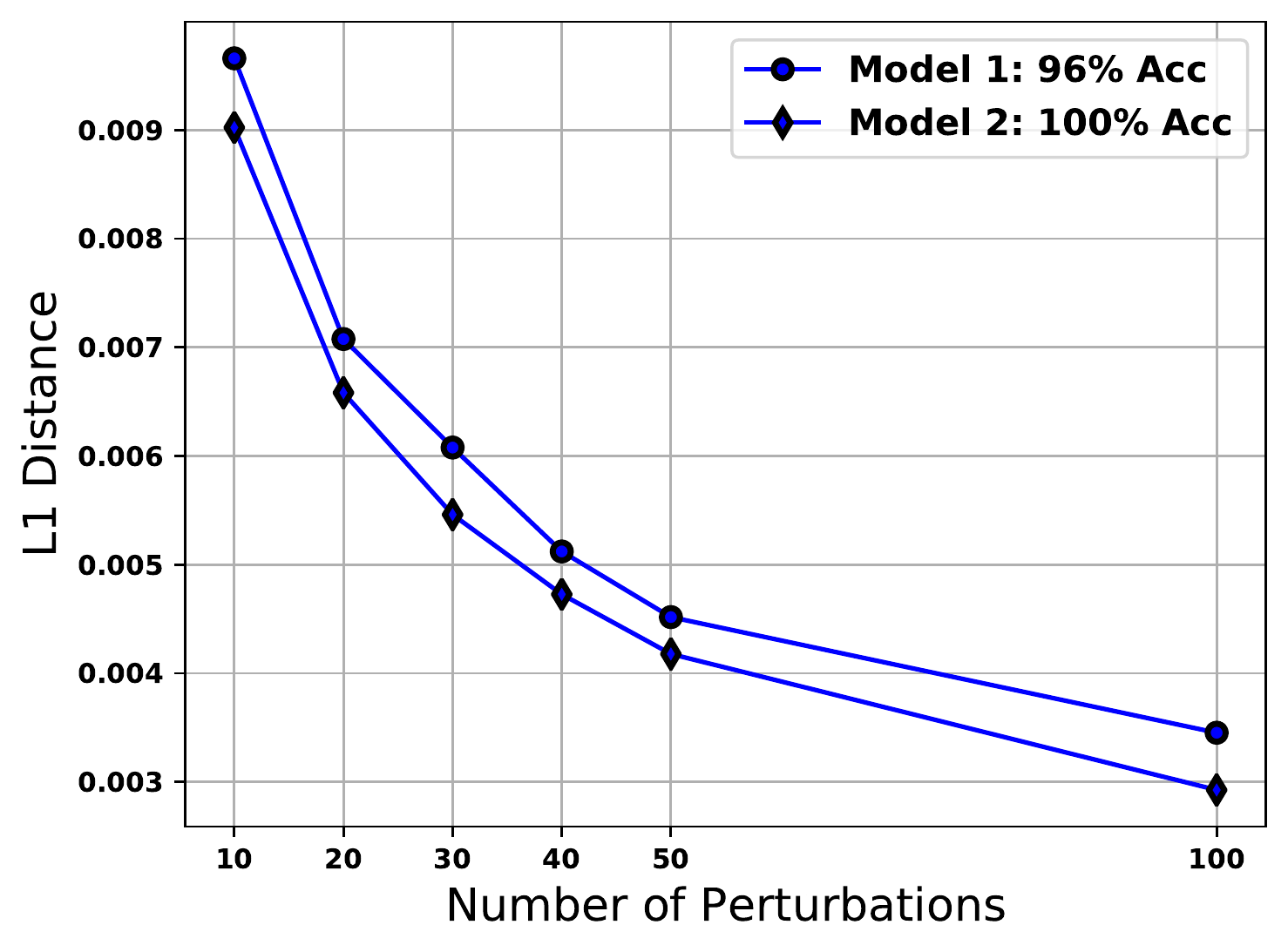}
    	\caption{Robustness of SmoothGrad}
        \label{fig:macc_sg_robust_sim}
	\end{subfigure}
	\begin{subfigure}{0.33\linewidth}
		\centering
    	\includegraphics[width=0.9\linewidth]{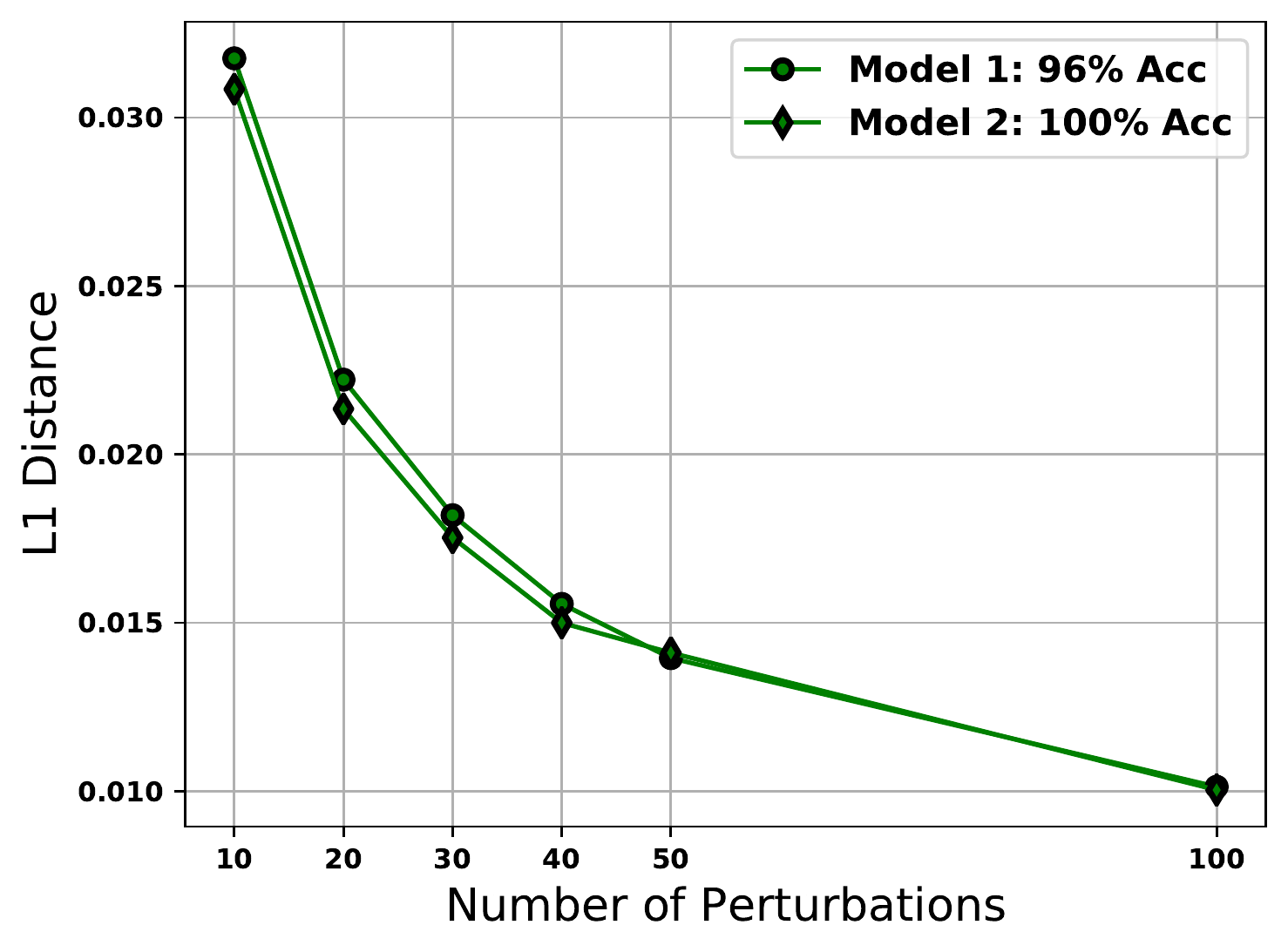}
    	\caption{Robustness of C-LIME}
        \label{fig:macc_lime_robust_sim}
	\end{subfigure}	
	\caption{Equivalence (\ref{fig:macc_equiv_sim}) and robustness plots for SmoothGrad~(\ref{fig:macc_sg_robust_sim}) and C-LIME~(\ref{fig:macc_lime_robust_sim}) for functions with various accuracies on the Simulated dataset. In each plot the Y axis corresponds to L1 distance and the X axis corresponds to the number of perturbations.}
	\label{fig:macc_sim}
\end{figure*}

\begin{figure*}[ht!]
	\begin{subfigure}{0.33\linewidth}
		\centering
    	\includegraphics[width=0.9\linewidth]{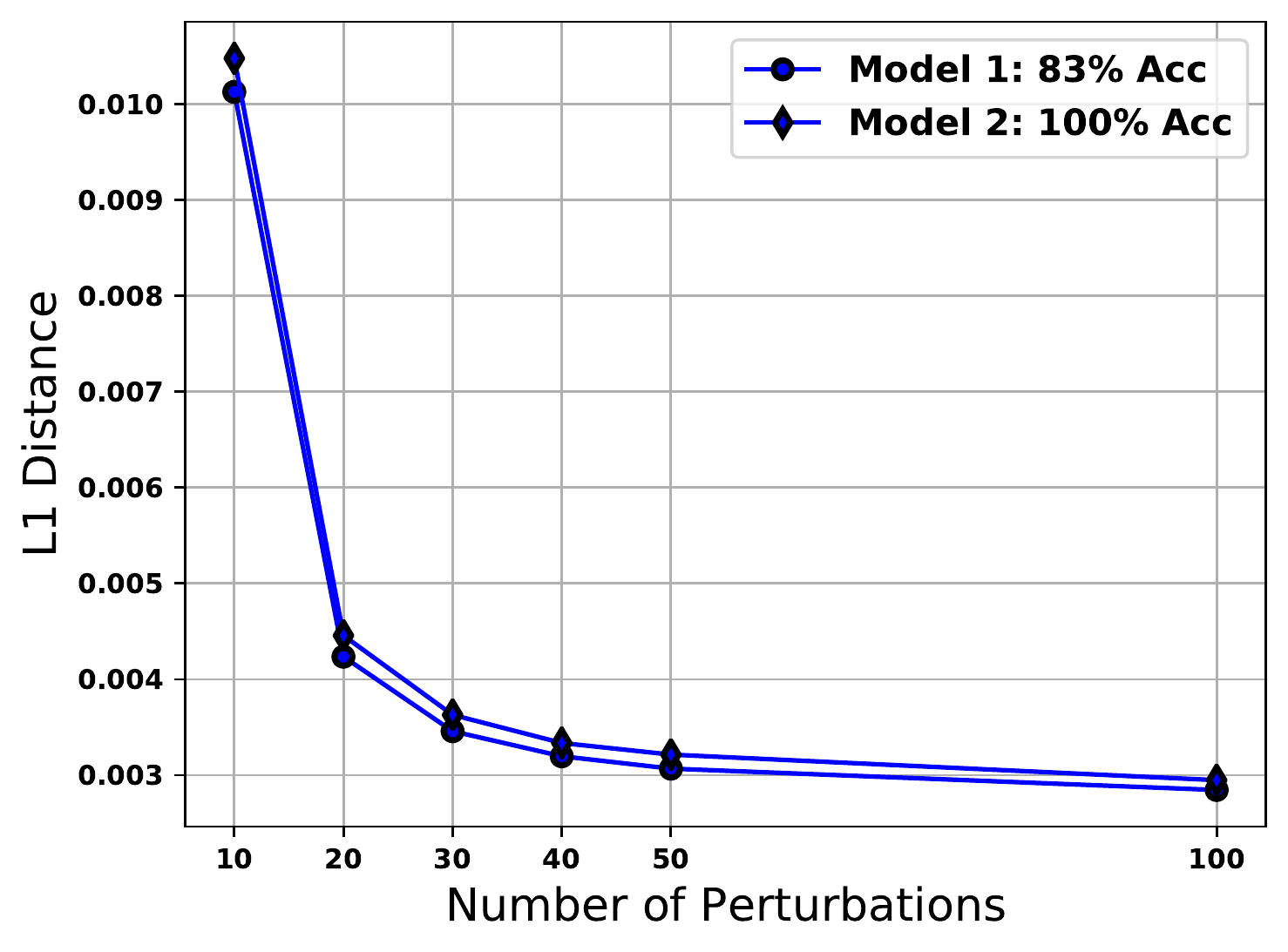}
    	\caption{Equivalence}
        \label{fig:macc_equiv_shopper}
	\end{subfigure}
	\begin{subfigure}{0.33\linewidth}
		\centering
    	\includegraphics[width=0.9\linewidth]{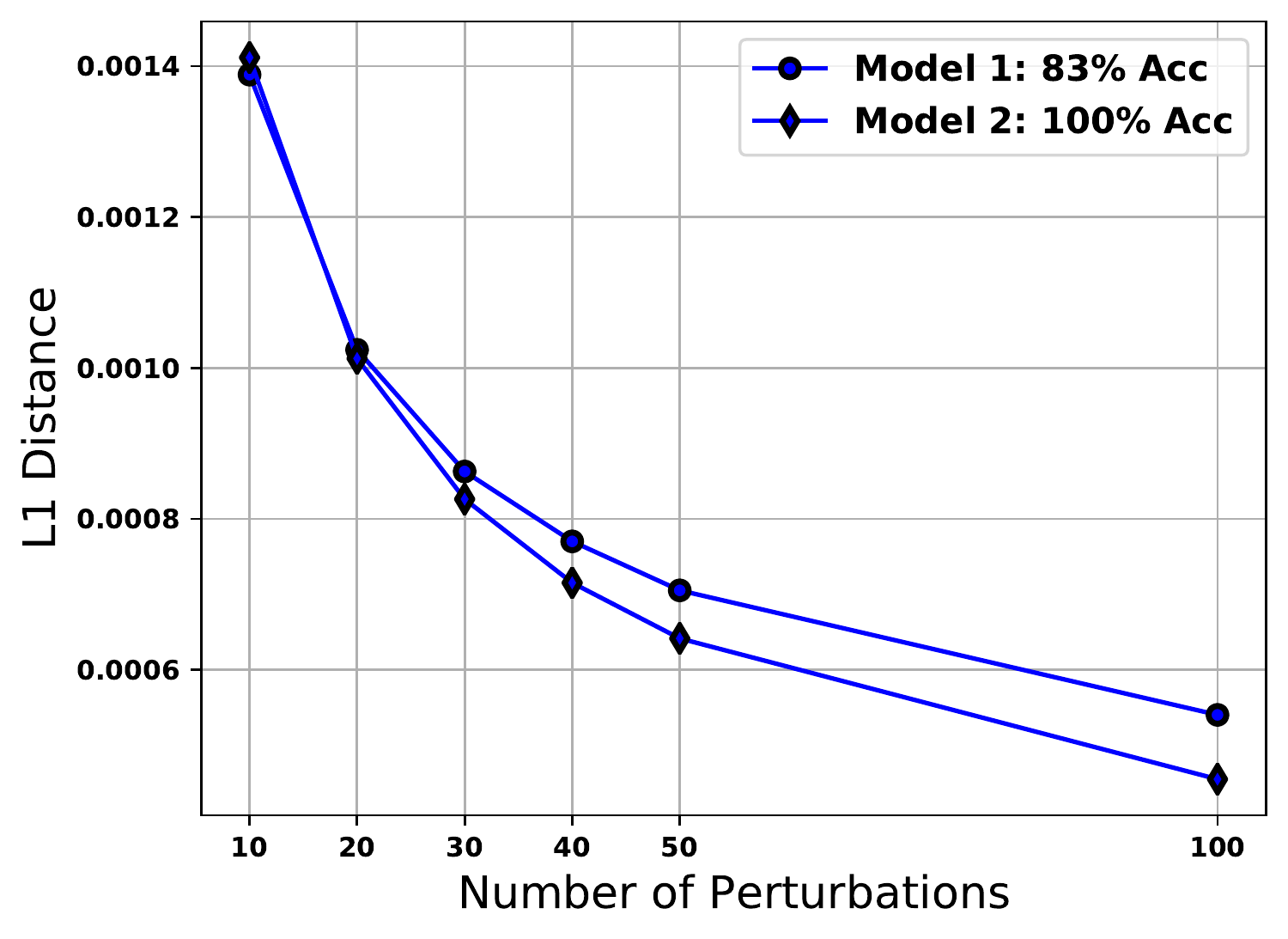}
    	\caption{Robustness of SmoothGrad}
        \label{fig:macc_sg_robust_shopper}
	\end{subfigure}
	\begin{subfigure}{0.33\linewidth}
		\centering
    	\includegraphics[width=0.9\linewidth]{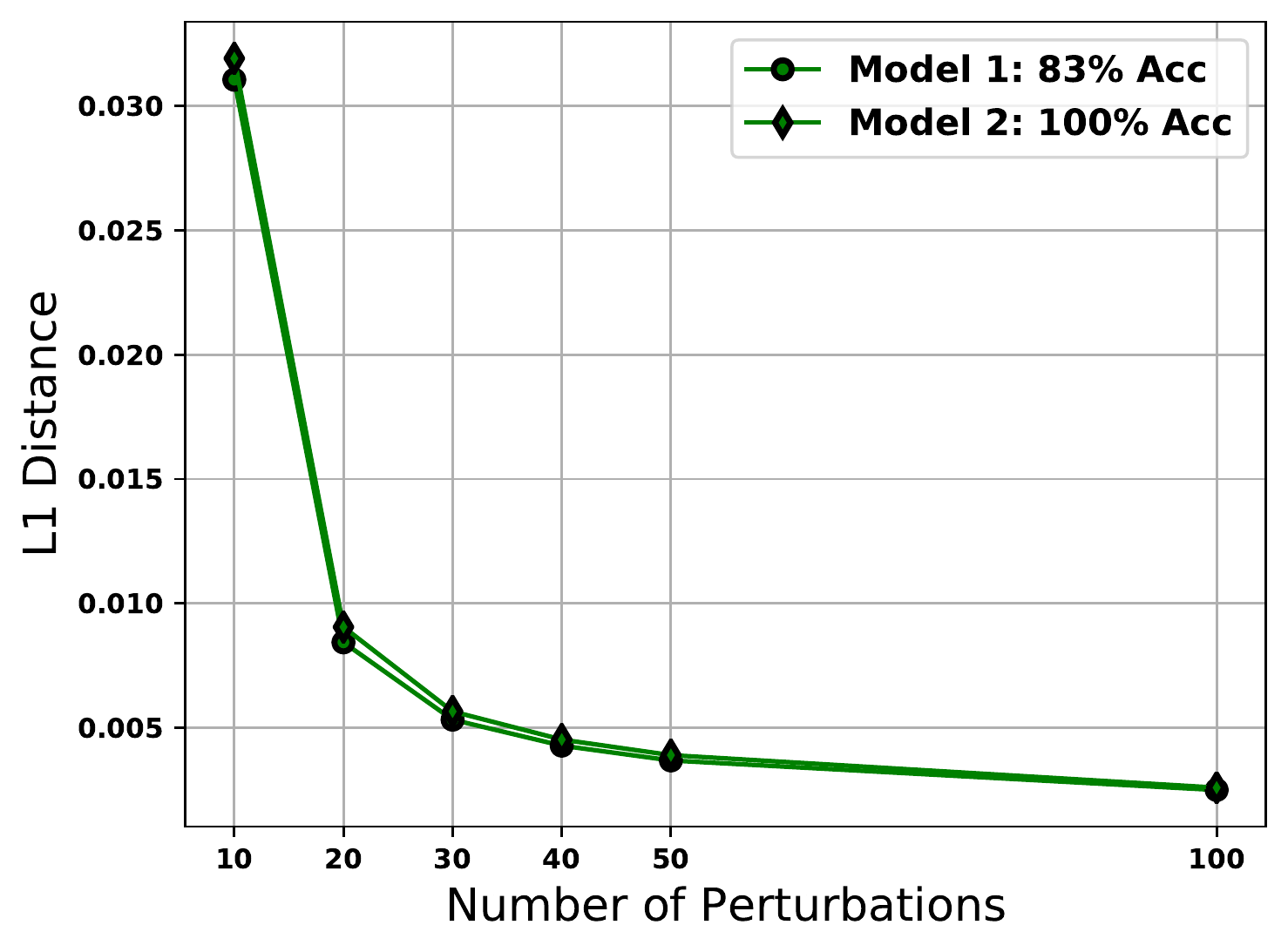}
    	\caption{Robustness of C-LIME}
        \label{fig:macc_lime_robust_shopper}
	\end{subfigure}	
	\caption{Equivalence (\ref{fig:macc_equiv_shopper}) and robustness plots for SmoothGrad~(\ref{fig:macc_sg_robust_shopper}) and C-LIME~(\ref{fig:macc_lime_robust_shopper}) for functions with various accuracies on the Online Shopping dataset. In each plot the Y axis corresponds to L1 distance and the X axis corresponds to the number of perturbations.}
	\label{fig:macc_shopper}
\end{figure*}
\fi

\end{document}